\documentclass[journal]{IEEEtran}
\usepackage{amsmath,amsfonts}
\usepackage{algorithmic}
\usepackage{array}
\usepackage[caption=false,font=normalsize,labelfont=sf,textfont=sf]{subfig}
\usepackage{textcomp}
\usepackage{stfloats}
\usepackage{url,soul}
\usepackage{comment}
\usepackage{verbatim}
\usepackage{graphicx}
\usepackage{arydshln}   

\usepackage[ruled,vlined]{algorithm2e}
\usepackage{color}
\usepackage[dvipsnames]{xcolor}

\def\BibTeX{{\rm B\kern-.05em{\sc i\kern-.025em b}\kern-.08em
    T\kern-.1667em\lower.7ex\hbox{E}\kern-.125emX}}
\usepackage{balance}

\usepackage[colorlinks,bookmarksopen,bookmarksnumbered,citecolor=black,urlcolor=blue,linkcolor=purple]{hyperref}

\newtheorem{thm}{Theorem}[section]

\newtheorem{lem}{Lemma}[section]

\newtheorem{assump}{Assumption}[section]

\newenvironment{proof}{\qquad \textit{Proof:}}{\hfill$\square$}

\newcommand{\boxend}{\hfill \ensuremath{\Box}}

\newcommand{\real}{{\mathbb{R}}}

\allowdisplaybreaks

\begin{document}

\title{A Computationally Efficient Maximum A Posteriori Sequence Estimation via Stein Variational Inference}


\author{Min-Won Seo and  Solmaz S. Kia, \emph{Senior Member, IEEE}



\thanks{Min-Won Seo, and Solmaz S. Kia are with the Department of Mechanical and Aerospace Engineering,
        University of California, Irvine, CA 92697, USA,
        {\tt\small \{minwons,solmaz\}@uci.edu}}
        }


\markboth{Journal of \LaTeX\ Class Files,~Vol.~00, No.~00, Month~Year}%
{How to Use the IEEEtran \LaTeX \ Templates}

\maketitle

\begin{abstract}
State estimation in robotic systems presents significant challenges, particularly due to the prevalence of multimodal posterior distributions in real-world scenarios. One effective strategy for handling such complexity is to compute maximum a posteriori (MAP) sequences over a discretized or sampled state space, which enables a concise representation of the most likely state trajectory. However, this approach often incurs substantial computational costs, especially in high-dimensional settings. In this article, we propose a novel MAP sequence estimation method, \textsf{Stein-MAP-Seq}, which effectively addresses multimodality while substantially reducing computational and memory overhead. Our key contribution is a sequential variational inference framework that captures temporal dependencies in dynamical system models and integrates Stein variational gradient descent (SVGD) into a Viterbi-style dynamic programming algorithm, enabling computationally efficient MAP sequence estimation. This integration allows the method to focus computational effort on MAP-consistent modes rather than exhaustively exploring the entire state space. \textsf{Stein-MAP-Seq} inherits the parallelism and mode-seeking behavior of SVGD, allowing particle updates to be efficiently executed on parallel hardware and significantly reducing the number of trajectory candidates required for MAP-sequence recursion compared to conventional methods that rely on hundreds to thousands of particles. We validate the proposed approach on a range of highly multimodal scenarios, including nonlinear dynamics with ambiguous observations, unknown data association with outliers, range-only localization under temporary unobservability, and high-dimensional robotic manipulators. Experimental results demonstrate substantial improvements in estimation accuracy and robustness to multimodality over existing estimation methods.

\end{abstract}


\vspace{-0.1in}
\section{Introduction}
\label{sec::intro}
\IEEEPARstart{S}{tate} estimation in robotic systems presents fundamental challenges, particularly when posterior distributions exhibit multimodality---a phenomenon that frequently arises in real-world applications due to nonlinear system-environment interactions, ambiguous observations, uncertain data associations, and multimodal sensor fusion. To address such complexity, Maximum A Posteriori (MAP) \emph{sequence} estimators offer a principled framework that not only handles non-Gaussian distributions but also leverages temporal dependencies inherent in dynamical systems to produce trajectory-consistent estimates. This paper focuses on the design of a computationally efficient and accurate MAP sequence (MAP-Seq) estimator that effectively manages multimodal posterior distributions while maintaining practical computational requirements for robotic applications.

MAP-Seq estimators are an extension of MAP estimators, which seek the most probable state in a given distribution. For multimodal distributions, MAP estimators are often preferred over Minimum Mean Square Error (MMSE) estimators, as MMSE estimators can produce estimates that lie between modes—potentially corresponding to low-probability regions that lead to erroneous state estimates. In contrast, MAP estimation seeks the most probable state, effectively capturing one of the dominant modes. MAP-Seq estimation extends this concept to find the most likely trajectory of states over a given time horizon, conditioned on the entire sequence of observations. By leveraging information across time to resolve ambiguities that may be present in point estimates, this approach offers more robust estimates, particularly in scenarios with temporary unobservability or when the system's dynamics provide strong temporal constraints. This framework has found applications across diverse domains, including image processing \cite{li2016maximum}, natural language processing \cite{kanda2017maximum}, and robotics for tasks such as high-level decision-making \cite{abdolmaleki2018maximum} and simultaneous localization and mapping (SLAM) or tracking \cite{islam2010multi,huang2015bank}.

While MAP-Seq estimation offers significant advantages, computing the optimal trajectory can be computationally expensive, especially for complex, high-dimensional systems. Common approaches use smoothing techniques, such as iterated Extended Kalman Smoothing (iEKS) and factor graph optimization~\cite{dellaert2017factor}, which iteratively estimate the trajectory under Gaussian assumptions and local linearization. However, these methods often struggle with multimodal distributions, as they are prone to converging to \emph{locally} optimal trajectories with poor initialization~\cite{huang2015bank}. In contrast, approaches based on finite (discrete or sampled) state spaces, such as the Viterbi algorithm, compute the \emph{globally} optimal trajectory via dynamic programming for MAP-Seq estimation~\cite{svensen2007pattern}, thereby avoiding local optima. This makes them particularly well-suited for estimation problems involving ambiguity, non-Gaussianity, or multimodality.

Despite this global optimality advantage, the Viterbi algorithm's computational complexity, which scales quadratically with the number of discrete states ($\mathcal{O}(N^2)$ for $N$ states), poses a significant bottleneck for real-time applications, especially when dealing with systems requiring fine-grained discretization of the state space for accurate representation. This limitation often restricts the Viterbi algorithm's applicability to coarse-grid discretizations or necessitates the use of high-performance computing resources and substantial memory allocation in large-scale scenarios.

To overcome these computational hurdles, researchers have explored various strategies aimed at simplifying the representation of the state space \cite{trogh2015advanced, seo2022online,scherhaufl2018blind,sun2019practical,bayoumi2019speeding,seo2025bayesian}. While such simplifications can lead to reduced computational demands, they often come at the cost of diminished estimation accuracy. Continuous-state-based Viterbi algorithms have been proposed \cite{champlin2000target,chigansky2011viterbi} to mitigate discretization errors, but these methods typically still grapple with significant computational complexity.

As an alternative to discrete-state Viterbi approaches, Particle Filter-based methods for MAP-Seq estimation have also been investigated \cite{godsill2001maximum,driessen2008particle,saha2009particle,nishida2009dynamic,morita2020fast}. While these methods have shown promise in handling complex multimodal distributions, they still rely on sequential Monte Carlo sampling, which requires large particle sets and suffers from resampling overhead and random sampling variability.

Recent advances in multimodal state estimation have explored deep learning-based paradigms, including amortized variational inference for nonlinear state-space models that learn neural networks to approximate filtering distributions~\cite{marino2018general}, neural particle filtering approaches that incorporate end-to-end learned models~\cite{jonschkowski2018differentiable}, and diffusion-based generative models for state estimation that perform posterior inference via iterative denoising processes~\cite{ye2025raggedi}. These methods exhibit strong representational capacity for complex and multimodal posteriors; however, they typically rely on offline training and implicit inference models, for which posterior densities are not explicitly available, making principled MAP-Seq estimation difficult.

The computational and representational challenges inherent in MAP-Seq estimation—including the limitations of discrete-state methods, sampling variability in particle filters, and implicit posteriors in learning-based approaches—motivate the need for alternative frameworks. Variational inference offers a promising avenue to address these challenges. Variational methods aim to approximate intractable posterior distributions with simpler, tractable forms by framing inference as an optimization problem~\cite{blei2017variational}. Unlike the implicit representations used in deep learning approaches, variational inference provides explicit posterior approximations that enable principled MAP-Seq computation. These techniques have gained attention in robotics and autonomous systems due to their ability to balance computational efficiency with the flexibility to capture complex, multimodal distributions~\cite{ala2013variational,gal2016uncertainty}. Notably, advances in variational inference have led to the development of methods that can handle non-Gaussian posteriors and temporal dependencies, making them particularly suitable for sequential estimation tasks~\cite{archer2015black,krishnan2017structured,seo2024sequential}. The application of variational inference to state estimation problems has shown promise in various domains, including visual-inertial odometry~\cite{mirchev2020variational}, SLAM~\cite{barfoot2020exactly}, and multi-target tracking~\cite{oh2020variational}.
\emph{Statement of contribution}: 
This paper presents \textsf{Stein-MAP-Seq}, a novel particle-based approach for MAP-Seq estimation in dynamical systems that addresses the challenges of multimodal state estimation while maintaining computational efficiency. Our method leverages Stein Variational Gradient Descent (SVGD)~\cite{liu2016stein}, a powerful non-parametric variational inference technique that deterministically evolves a set of particles to serve as samples from the target distribution, minimizing the Kullback-Leibler divergence between the particle-based empirical distribution and the true posterior. Unlike traditional Monte Carlo methods that rely on stochastic sampling and/or resampling procedures, SVGD deterministically transports particles toward high-probability regions using functional gradients that collectively capture the structure of complex, potentially multimodal distributions.

 The key innovation of \textsf{Stein-MAP-Seq} lies in leveraging SVGD-generated particles as representative samples to transform the intractable continuous MAP optimization into an efficient discrete combinatorial problem. \textsf{Stein-MAP-Seq} embeds SVGD into a Viterbi-style dynamic programming algorithm within a sequential variational inference framework that explicitly models the temporal dependencies inherent in dynamical systems. This novel combination leverages the mode-seeking yet repulsive behavior of SVGD to represent and separate multiple dominant modes, enabling efficient MAP sequence estimation via dynamic programming over the resulting discrete sample space. While the forward recursion incurs $\mathcal{O}(M^2)$ computational complexity per time step, where $M$ denotes the number of particles, \textsf{Stein-MAP-Seq} achieves substantial efficiency gains by requiring far fewer particles ($M \leq 40$) compared to conventional particle filter-based MAP-Seq estimators that require hundreds to thousands of particles ($N \gg M$), resulting in dramatically lower computational cost in practice. Furthermore, since the particle updates in SVGD can be computed independently across samples, its inherent parallelism makes \textsf{Stein-MAP-Seq} naturally amenable to efficient implementation on modern parallel computing architectures, which are increasingly available in robotic platforms~\cite{barcelos2021dual,wang2021variational,maken2021stein,maken2022stein,heiden2022probabilistic}.

Our approach fundamentally differs from traditional sampling-based methods by using SVGD's deterministic particle evolution to generate high-quality samples that are then used as discrete states in a Viterbi-style dynamic programming algorithm. We rigorously evaluate the proposed method on challenging scenarios, including nonlinear dynamics with ambiguous measurements, pose estimation under unknown data associations with outliers, range-only localization under temporary unobservability, and high-dimensional manipulators with kinematic redundancy, demonstrating substantial improvements in estimation accuracy and robustness to multimodality over state-of-the-art baselines.

The remainder of this paper is organized as follows: Section~\ref{sec::ProbDef} formally defines the problem and presents our assumptions. Section~\ref{sec::Back} provides a brief overview of SVGD. Section~\ref{sec::method} details the proposed \textsf{Stein-MAP-Seq} algorithm. Section~\ref{sec::results} presents estimation results across diverse scenarios and compares performance with existing methods. Finally, Section~\ref{sec::conclusion} concludes the paper. Proofs and auxiliary results are provided in the appendices.

\section{Problem Definition}
\label{sec::ProbDef}
Consider a discrete-time (fixed interval) dynamical system whose state transition probability is modeled as a Markov process~\eqref{eq::SSM-x} and its observation likelihood function~\eqref{eq::SSM-z} given~by:
\begin{subequations}\label{eq::SSM}
\begin{align}
    x_t & \sim p(x_t|x_{t-1}), \label{eq::SSM-x}\\
    z_t & \sim p(z_t|x_t),\label{eq::SSM-z}
\end{align}
\end{subequations}
where $x_t \in \real^{n_x}$ denotes the state of the system and $z_t \in \real^{n_z}$ denotes the observation data at time $t \in \mathbb{N}_{>0}$. We assume that the prior distribution $p(x_0)$ is constructed from available prior information. The dynamical model and observation data adhere to the following assumptions.

\begin{assump} \label{assump:smoothness}{(Smoothness and regularity of the dynamical model). Under the dynamical model~\eqref{eq::SSM}, $\log p(x_t|x_{t-1})$ and $\log p(z_t|x_t)$ are continuously differentiable with respect to $x_t$, with Lipschitz continuous gradients.}\boxend
\end{assump}

\begin{assump}\label{assum::obs_indep}{(Assumptions on $z_t$). The sensor observations at each time step, $z_t$,  are independent and identically distributed (i.i.d.). Moreover, there are no conditional dependencies between observations from different sensors.}\boxend
\end{assump}

\medskip
Our objective in this paper is to compute the MAP sequence of states, denoted as $x_{0:T}^{\text{MAP}}$, for the dynamical system~\eqref{eq::SSM} from the  sequence of observations from time $1$ to the final time $T$, i.e., $z_{1:T} = \{z_1, \dots, z_T\}$, given by
\begin{align}  
\label{eq::MAP_est}  
    x_{0:T}^{\text{MAP}} = \arg\!\max_{\!\!\!\!x_{0:T}} p(x_{0:T} | z_{1:T}).  
\end{align} 

Solving~\eqref{eq::MAP_est} is computationally intractable, particularly when $p(x_{0:T} | z_{1:T})$ exhibits multimodality and/or lacks a closed-form expression, making exact analytical computation impossible. To address this challenge, we propose a \emph{sequential variational inference} framework to obtain an approximate MAP sequence solution. The details of the proposed approach are presented in Section~\ref{sec::method}. Before presenting those details, the next section provides background on the Stein Variational Gradient Descent method, which serves as a key component of our approach.

\begin{algorithm}[t]
\caption{Stein Variational Gradient Descent}
\label{alg:svgd_orig}
\textbf{Input:} The score function $\nabla_x \log p(x)$.

\textbf{Goal:} A set of particles $\{x_i\}_{i=1}^{N_s}$ that approximates $p(x)$.

\textbf{Initialize} a set of particles $\{x_i^{(0)}\}_{i=1}^{N_s}$; choose a positive definite kernel $\kappa(x, x')$ and step-size $\epsilon$.\\
\textbf{For} iteration \textit{l} \textbf{do}
{
\begin{align*}
    x_i^{(l+1)} \leftarrow x_i^{(l)} + \epsilon \hat{\phi}^{\star}(x_i^{(l)}) \quad \forall \, i \in\{ 1, \ldots, N_s\} 
\end{align*}}
where $\hat{\phi}^{\star}(x_i^{(l)})$ is given in~\eqref{eq::svgd_iterate}.
\end{algorithm}

\section{A Brief Overview of Stein Variational Gradient Descent method}\label{sec::Back}

Variational inference (VI) approximates a target distribution $p(x)$ using a simpler distribution $q^{\star}(x)$ found in a predefined set $\mathcal{Q}$ by minimizing the Kullback–Leibler (KL)\footnote{KL divergence is defined as $\text{KL}[q(x) \parallel p(x)] = \int q(x) \log \frac{q(x)}{p(x)} dx$.}  divergence~\cite{kullback1951information}, i.e., 
\begin{align}
 \label{eq::VI_KLD}
    q^{\star}(x) = \arg\!\min_{\!\!\!\!\!\!q \in \mathcal{Q}} \text{KL} \bigr[q(x)\parallel p(x) \bigr], 
\end{align}
where the choice of the model space $\mathcal{Q}$ is critical for balancing accuracy and computational efficiency; however, achieving this balance remains challenging.

Stein Variational Gradient Descent (SVGD) offers a non-parametric approach by iteratively transporting a set of initial samples $\{x_i^{(0)}\}_{i=1}^{N_s}$ towards $p(x)$ using a perturbation function $\phi(x)$: 
\begin{align}
    x_i^{(l+1)} = x_i^{(l)} + \epsilon \phi(x_i^{(l)}). \nonumber
\end{align}
The optimal perturbation direction $\phi^{\star}(x) \in \mathcal{H}^d$ that minimizes the KL divergence at each step has a closed form derived from the Stein's identity\footnote{$\phi^{\star}(x)$ is chosen from the unit ball of a vector-valued RKHS, $\mathcal{H}^d$.}:
\begin{align}
    \phi^{\star}(x) = \mathbb{E}_{x' \sim q(x)} \left[ \kappa(x, x') \nabla_{x'} \log p(x') + \nabla_{x'} \kappa(x, x') \right], \nonumber
\end{align}
where $\kappa(x, x')$ is a positive definite kernel~\cite{liu2016stein}. In practice, the perturbation direction is computed empirically using the samples $\{x_i^{(l)}\}_{i=1}^{N_s}$:
\begin{align}
\label{eq::svgd_iterate}
    \hat{\phi}^{\star}(x_i^{(l)}) =\frac{1}{N_s}\sum\nolimits_{k=1}^{N_s} \Bigl( \kappa(x_i^{(l)}, x_k^{(l)}) &\nabla_{x_k^{(l)}} \log p(x_k^{(l)})\, +\nonumber\\
    &\nabla_{\!\!x_k^{(l)}} \kappa(x_i^{(l)}, x_k^{(l)}) \Bigl). 
    \end{align}
The update $x_i^{(l+1)} = x_i^{(l)} + \epsilon \hat{\phi}^{\star}(x_i^{(l)})$ iteratively pushes particles toward $p(x)$, enabling expectation approximations via the empirical mean. Algorithm~\ref{alg:svgd_orig} summarizes the SVGD algorithm.

The samples generated by SVGD, denoted as $\{x_i\}_{i=1}^{N_s}$, serve a dual purpose in variational inference. First, they provide an empirical approximation of the optimal distribution $q^\star(x)$ in~\eqref{eq::VI_KLD} as $q^\star(x)= \frac{1}{N_s}\sum\nolimits_{i=1}^{N_s} \delta(x - x^i)$. Second, and equally important, these particles are commonly interpreted as representative samples from the target distribution $p(x)$ itself. Unlike traditional sampling methods such as MC or MCMC, which generate independent samples through stochastic processes, SVGD produces a deterministic set of interacting particles that collectively capture the structure of complex, potentially multimodal distributions through optimization-driven particle transport. This deterministic mechanism fundamentally differs from the random sampling variability inherent in Monte Carlo methods, offering computational advantages for downstream tasks, including efficient parallel computation and natural handling of multimodal posterior structures. Consequently, SVGD has emerged as a powerful sampling tool for approximate inference in high-dimensional and multimodal scenarios where traditional sampling methods face significant challenges.

Conventional SVGD addresses stationary inference problems, where the target distribution does not evolve over time. In contrast, state estimation in dynamical systems requires sequential inference over trajectories, where temporal dependencies between states must be preserved. In the next section, we extend SVGD to this sequential setting by developing a framework that leverages SVGD's particle transport to approximate time-evolving conditional distributions while maintaining the Markov structure of dynamical systems, enabling efficient MAP sequence estimation via dynamic programming.

\section{Stein Maximum A Posteriori Sequence Estimation}\label{sec::method}
In this section, we develop a MAP sequence estimator for the dynamical system~\eqref{eq::SSM} by integrating variational inference with dynamic programming. Computing the MAP trajectory $x_{0:T}^\mathrm{MAP}$ requires constructing an optimal distribution $q^{\star}(x_{0:T})$ whose support is concentrated on the (unknown) global maxima of the true posterior $p(x_{0:T}|z_{1:T})$, and subsequently identifying these modes. Formally, the problem can be stated as:
\begin{align}
    \label{eq::MAP_est_formulation}
 &\Bigl(q^{\star}(x_{0:T}),~x_{0:T}^\mathrm{MAP}\Bigl) =\\
 &~~~~~\Bigl(\arg\min_{\!\!\!\!\!\!q(x_{0:T})}\text{KL}[q(x_{0:T})\|p(x_{0:T}|z_{1:T})],~
 \underset{x_{0:T}}{\arg\max}\,q^{\star}(x_{0:T}) \Bigl). \nonumber
\end{align}

Solving these optimization problems exactly presents significant challenges. When the dynamical system~\eqref{eq::SSM} is nonlinear, the resulting posterior $p(x_{0:T}|z_{1:T})$ becomes non-Gaussian and often multimodal, making closed-form representation impractical and exact analytical computation impossible. Moreover, the choice of $q(x_{0:T})$ significantly impacts estimation performance~\cite{godsill2001maximum}. In the second optimization problem, finding the mode of $q^{\star}(x_{0:T})$ is intractable, since $q^{\star}(x_{0:T})$ typically resides on a high-dimensional continuous support.

To address these challenges, we propose \textsf{Stein-MAP-Seq}, a MAP-sequence estimator for dynamical systems~\eqref{eq::SSM} that preserves temporal dependencies via a sequential variational inference framework (Lemma~\ref{lem::factKL}). \textsf{Stein-MAP-Seq} integrates SVGD into a Viterbi-style dynamic programming algorithm, where SVGD's deterministic gradient flow generates a compact discrete state representation, enabling efficient global MAP trajectory recovery through forward recursion and backtracking.

Our approach transforms the intractable continuous optimization in~\eqref{eq::MAP_est_formulation} into a tractable discrete problem through a two-stage strategy: first, SVGD generates particle-based approximations of the sequential posterior distributions, discretizing the continuous state space into a finite sample support; second, dynamic programming over this discrete particle support efficiently identifies the globally optimal MAP trajectory. This particle-based discretization fundamentally transforms the optimization from an intractable continuous search to an efficient discrete combinatorial problem.


The remainder of this section presents the theoretical foundations and algorithmic steps of our proposed \textsf{Stein-MAP-Seq}. We develop a series of marginalized conditional KL objectives (Lemma~\ref{lem::factKL}) that adapt SVGD from stationary inference to sequential trajectory estimation, preserving temporal dependencies through the Markov structure of dynamical systems. We then formalize the integration of these SVGD-generated particle approximations with Viterbi-style dynamic programming (Theorem~\ref{thm::ELBO_MAP}), establishing the theoretical guarantees for MAP sequence recovery and analyzing computational complexity.

\subsection{Sequential Variational Inference for Dynamical Systems}
To make the complex optimization problems in~\eqref{eq::MAP_est_formulation} tractable and derive a computable solution for $q^{\star}(x_{0:T})$, we begin by performing a series of strategic manipulations on the KL divergence objective
\begin{equation}\label{eq::MAP_KLD}
\text{minimize}~~\text{KL} \bigr[q(x_{0:T})\parallel p(x_{0:T}|z_{1:T}) \bigr].
\end{equation}

To facilitate a computable solution for VI optimization problems involving conditional posteriors based on measurements/data, it is common practice to reformulate the objective in terms of the Evidence Lower BOund (ELBO). In the case of~\eqref{eq::MAP_KLD}, this relationship is given by:
\begin{align} \label{eq::KLD_MAP_ELBO}
  \!\!\!\text{KL} \bigr[q(x_{0:T})\parallel p(x_{0:T}|z_{1:T}) \bigr] \!=\! -\mathcal{L}(x_{0:T}) + \log p(z_{1:T}),
\end{align}
where
\begin{equation}\label{eq::ELBO_term}
 \mathcal{L}(x_{0:T})= \int q(x_{0:T})\, \log\frac{p(x_{0:T}, z_{1:T})}{q(x_{0:T})} dx_{0:T}
\end{equation}
is known as the ELBO; for a detailed derivation, see Lemma~\ref{lem::ELBO} in Appendix~\ref{sec::appen_A}. The KL divergence in~\eqref{eq::KLD_MAP_ELBO} is always nonnegative and equals zero if and only if $q(x_{0:T}) = p(x_{0:T} \mid z_{1:T})$. Moreover, $\log p(z_{1:T})$ is constant with respect to $x$ and does not affect the optimization. Therefore, minimizing the KL divergence, as shown on the left-hand side of~\eqref{eq::KLD_MAP_ELBO}, is equivalent to maximizing the ELBO $\mathcal{L}(x_{0:T})$ defined in~\eqref{eq::ELBO_term}. Consequently, we reformulate the objective as the ELBO maximization problem, given by
\begin{align}
\label{eq::Stein_MAP_opt1}
    q^{\star}(x_{0:T}) = \arg \!\! \max_{\!q(x_{0:T})} \mathcal{L}(x_{0:T}),
\end{align}
where solving the optimization problem requires evaluating both the log joint probability $\log p(x_{0:T}, z_{1:T})$ and the proposal distribution $q(x_{0:T})$ within the ELBO.

To compute the log joint probability $\log p(x_{0:T}, z_{1:T})$, we factorize it based on the first-order Markov property of the dynamical system~\eqref{eq::SSM-x} and Assumption~\ref{assum::obs_indep} for observations $z_{1:T}$, yielding:
\begin{align}
\label{eq::Joint_prob}
    \log p(x_{0:T}, z_{1:T}) =&\sum\nolimits_{t=1}^T \log p(z_t|x_t) +\\&\sum\nolimits_{t=1}^T \log p(x_t|x_{t-1}) +  \log p(x_0). \nonumber
\end{align}

Using~\eqref{eq::Joint_prob}, the ELBO in~\eqref{eq::ELBO_term} reads as
\begin{align}
\label{eq::ELBO}
      \mathcal{L}(x_{0:T}) &= \int q(x_{0:T}) \log p(x_0) dx_{0:T} \nonumber \\
      &~~~+ \int q(x_{0:T}) \sum\nolimits_{t=1}^T \log p(x_t|x_{t-1}) dx_{0:T} \nonumber \\
      &~~~+ \int q(x_{0:T}) \sum\nolimits_{t=1}^T \log p(z_t|x_t) dx_{0:T} \nonumber \\ 
      &~~~- \int q(x_{0:T})\log q(x_{0:T}) dx_{0:T}.
\end{align}

The VI literature employs two broad techniques to specify the joint proposal distribution $q(x_{0:T})$ in the ELBO: selecting a manageable parametric form~\cite{blei2017variational,barfoot2020exactly,seo2024sequential}, or employing the mean field approximation~\cite{svensen2007pattern}. The first approach involves using a parameterized distribution, such as a Gaussian distribution. In the second approach, the joint proposal distribution is factorized into independent distributions for each time variable, i.e., $q(x_{0:T}) = \prod_{t=0}^{T}q(x_t)$. These simplifications make the optimization process more tractable but may reduce approximation accuracy due to the inflexible distributional form and the loss of temporal dependencies.

Unlike existing techniques, in this study, we adopt a non-parametric method based on SVGD, which provides flexibility in approximating complex posterior structures. In addition, we impose a first-order Markov property on the proposal distribution to preserve temporal dependencies between adjacent states, an assumption that is naturally aligned with the sequential nature and underlying dynamics of the system, following the approach inspired by~\cite{courts2023variational} as follows.

\begin{assump}
\label{assump::MAP}{(First-order Markov property on proposal distribution). The proposal distribution $q(x_{0:T})$ in~\eqref{eq::Stein_MAP_opt1} is chosen such that
 \begin{align}
\label{eq::Assumption_proposal}
 q(x_t|x_{t-1}) = q(x_t|x_{t-1},x_{0:t-2}).
 \end{align}}
\end{assump}

Under Assumption~\ref{assump::MAP}, the following lemma establishes the \textbf{key structural result} that distinguishes our approach from standard VI methods.

\begin{lem}
\label{lem::proposal} {(The form of optimal proposal distribution). Let Assumption~\ref{assump::MAP} hold. Then, the optimal proposal distribution $q(x_{0:T})$ in~\eqref{eq::Stein_MAP_opt1} is factorized as
 \begin{align}
 \label{eq::Proposal_distribution}
 q(x_{0:T}) = q(x_0)\prod\nolimits_{t=1}^{T} q(x_t|x_{t-1}).
 \end{align}}
\end{lem}
\begin{proof}
 See Appendix~\ref{sec::appen_B}.
\end{proof}

\textbf{Significance of Lemma~\ref{lem::proposal}:} This result is \textbf{not} a standard VI factorization—it establishes a sequential conditional structure that fundamentally differs from both parametric VI and mean field approximations. Unlike mean field approaches that sacrifice temporal structure through independence ($q(x_{0:T}) = \prod_t q(x_t)$), the factorization in~\eqref{eq::Proposal_distribution} \emph{preserves temporal dependencies} through conditional distributions $q(x_t|x_{t-1})$ while enabling \emph{sequential recursive computation}. This structure is essential for: (i) justifying the SVGD particle updates in Eq.~\eqref{eq::SVGD_newupdate} that condition on particles from previous time steps, (ii) deriving the factorized ELBO in Lemma~\ref{lem::factKL} with marginalized conditional KL terms, and (iii) enabling Viterbi-style dynamic programming over time-evolving particle distributions.

According to Lemma~\ref{lem::proposal}, the proposal distribution $q(x_{0:T})$ in~\eqref{eq::ELBO} admits this factorized structure. Consequently, the ELBO can be expressed as a summation of factorized KL divergence terms, each marginalized with respect to the previous optimal distribution, as stated in the following result.

\begin{lem}\label{lem::factKL} {(The factorization of marginalized conditional KL divergence). Let Assumption~\ref{assump:smoothness},~\ref{assum::obs_indep}, and~\ref{assump::MAP} hold. Given the optimal proposal distribution~\eqref{eq::Proposal_distribution}, the ELBO in~\eqref{eq::Stein_MAP_opt1} is factorized as
\begin{align}
\label{eq::ELBO_factorized}
 &\mathcal{L}(x_{0:T}) = -\text{KL} \bigr[q(x_0) \parallel p(x_0)\bigr] \nonumber \\
 &- \int \text{KL} \bigr[q(x_1|x_0) \parallel p(z_1,x_1|x_0)\bigr] q(x_0) dx_0 \\
&- \sum_{t=2}^T\! \int\!\! \text{KL} \bigr[q(x_t|x_{t-1})\!\!\parallel\! p(z_t,x_t|x_{t-1})\bigr] q(x_{t-1}|x_{t-2}) dx_{t-1}\!,\! \nonumber
\end{align}}
where $p(z_t,x_t|x_{t-1})=p(z_t|x_t)p(x_t|x_{t-1})$.
\end{lem}
\medskip
\begin{proof}
 See Appendix~\ref{sec::appen_C}.
\end{proof}

\bigskip

\subsubsection{SVGD-Based Sequential Approximation of Optimal Distributions}
The factorized ELBO structure in Lemma~\ref{lem::factKL} reveals that the optimal proposal distribution can be computed sequentially by solving a series of marginalized conditional KL minimization problems. However, computing these optimal conditional distributions $q^{\star}(x_t|x_{t-1})$ analytically remains intractable due to the complex, potentially multimodal nature of the target distributions $p(z_t,x_t|x_{t-1})$. 

To address this computational challenge, we leverage SVGD's non-parametric approximation capabilities to sequentially construct particle-based representations of $q^{\star}(x_t|x_{t-1})$ for each time step. This approach transforms the continuous optimization problem into a manageable particle-based computation while preserving the temporal dependencies captured in the factorized ELBO.

The sequential particle generation procedure is as follows:
\begin{enumerate}
\item \textbf{Initialization ($t=1$):} Given the initial state $x_0^{\text{true}}$ is known, we model its prior as a sharply concentrated Gaussian, $p(x_0)=\mathcal{N}(x_0; x_0^{\text{true}}, \varepsilon\mathbf{I})$, $\varepsilon \ll 1$, then $q^{\star}(x_0) = p(x_0)$\footnote{This setup is essential for Theorem~\ref{thm::ELBO_MAP}.}. For $t=1$, given $x_0^{\text{true}}$, compute the distribution $q^{\star}(x_1|x_0) = \frac{1}{N_s} \sum\nolimits_{i=1}^{N_s} \delta(x_1 - x_1^i)$, where particles $\{x_1^i\}_{i=1}^{N_s}$ are updated using SVGD with the following update direction:
\begin{align} \label{eq::SVGD_initial_newupdate}
     &\hat{\phi}^{\star}(x_1^i) = \frac{1}{N_s}\! \sum_{k=1}^{N_s} \Bigl[\kappa(x_1^i, x_1^k) \nabla_{x_1^k} \! \log p(z_1, x_1^k|x_0^{\text{true}}) \nonumber \\
     &\qquad\qquad\qquad\qquad~~~~ + \nabla_{x_1^k} \kappa(x_1^i, x_1^k)\Bigl].
\end{align}

\item \textbf{Sequential Particle Update ($t > 1$):} For $2 \leq t \leq T$, given the set of particles $\{x_{t-1}^i\}_{i=1}^{N_s}$ that empirically represent $q(x_{t-1}|x_{t-2}) \approx \frac{1}{N_s}\sum\nolimits_{i=1}^{N_s} \delta(x_{t-1} - x_{t-1}^i)$, we aim to find a new set of particles $\{x_{t}^i\}_{i=1}^{N_s}$. These new particles should empirically estimate the optimal conditional distribution $q^\star(x_t|x_{t-1})$, which is the solution to the following minimization problem:
    \begin{align}\label{eq::sequential_q_estimator}
    q^\star(x_t|x_{t-1})=&\arg\!\!\!\!\min_{q(x_t|x_{t-1})} \int \text{KL} \bigr[q(x_t|x_{t-1})\\
    &\!\parallel \!p(z_t,x_t|x_{t-1})\bigr] \nonumber\times q(x_{t-1}|x_{t-2})\, \text{d}x_{t-1}
    \end{align}
Lemma~\ref{lem::conditional_svgd} below presents the specific perturbation direction required to evolve the particles $\{x_{t-1}^i\}_{i=1}^{N_s}$ to $\{x_{t}^i\}_{i=1}^{N_s}$, thereby empirically fitting $q^\star(x_t|x_{t-1})$ as defined in~\eqref{eq::sequential_q_estimator}.
\end{enumerate}

\begin{lem}
\label{lem::conditional_svgd}{(KL descent of marginalized conditional SVGD update). Let Assumption~\ref{assump:smoothness} and~\ref{assump::MAP} hold, and define the marginalized KL divergence objective at time $t$ as 
\begin{align}
    \mathcal{J}_t(q) \!\!=\!\!\! \int \!\!\text{KL} \bigr[q(x_t|x_{t-1}) \!\!\parallel\! p(z_t, x_t|x_{t-1})\!\bigr] q(x_{t-1}|x_{t-2}) dx_{t-1}\!.\!\! \nonumber
\end{align}
Given a set of particles $\{x_{t-1}^j\}_{j=1}^{N_s} $ representing the distribution $q(x_{t-1}|x_{t-2})$, let the SVGD update direction be defined as 
\begin{align} \label{eq::SVGD_newupdate}
    &\!\hat{\phi}^{\star}(x_t^i) = \!\frac{1}{N_s^2}\! \sum_{k=1}^{N_s} \! \sum_{j=1}^{N_s} \Bigl[\kappa(x_t^i, x_t^k) \nabla_{x_t^k} \! \log p(z_t, x_t^k|x_{t-1}^j) \nonumber \\
    &\qquad\qquad\qquad\qquad~~~~ + \nabla_{x_t^k} \kappa(x_t^i, x_t^k)\Bigl],
\end{align}
and apply the particle update:
\begin{align}
    x_t^i \leftarrow x_t^i + \epsilon \, \hat{\phi}^{\star}(x_t^i). \nonumber
\end{align}
so that
\begin{align}
    q_\epsilon(x_t | x_{t-1}) = \frac{1}{N_s} \sum\nolimits_{i=1}^{N_s} \delta\Bigl(x_t - (x_t^i + \epsilon \, \hat{\phi}^{\star}(x_t^i))\Bigl). \nonumber
\end{align}
Then, under assumption of the boundedness of the kernel $\kappa(\cdot, \cdot)$, the marginalized KL objective satisfies
\begin{align}
    \left. \frac{d}{d\epsilon} \mathcal{J}_t(q_\epsilon) \right|_{\epsilon = 0} < 0, \nonumber
\end{align}
unless $q(x_t|x_{t-1}) = p(z_t, x_t|x_{t-1})$ almost surely under $q(x_{t-1}|x_{t-2})$.}
\end{lem}
\begin{proof}
    See Appendix~\ref{sec::appen_D}.
\end{proof}

\medskip
This sequential SVGD procedure generates particle sets $\{x_t^i\}_{i=1}^{N_s}$ for each time step $t \in \{1, \ldots, T\}$ that collectively provide a finite discrete approximation of the optimal proposal distribution $q^{\star}(x_{0:T})$. The resulting particle-based representation transforms the intractable continuous MAP optimization problem into a discrete combinatorial problem that can be efficiently solved using dynamic programming, as detailed in the following subsection.

\subsection{Particle-Based Dynamic Programming for MAP Sequence Recovery}\label{subsec::SteinMAPDP}

The sequential SVGD procedure described in the previous subsection generates particle sets $\{x_t^i\}_{i=1}^{N_s}$ for each time step $t \in \{1, \ldots, T\}$ that empirically represent the optimal conditional distributions $q^{\star}(x_t|x_{t-1}) \approx \frac{1}{N_s}\sum\nolimits_{i=1}^{N_s} \delta(x_t - x_t^i)$. These particles serve as samples from the optimal distributions, providing a finite discrete approximation of the continuous state space at each time step. This particle-based representation transforms the continuous MAP optimization problem in~\eqref{eq::MAP_est_formulation} into a discrete combinatorial problem over finite particle sets.

Since $\text{KL} \bigr[q^{\star}(x_{0:T})\parallel p(x_{0:T}|z_{1:T}) \bigr] \rightarrow 0$ implies $q^{\star}(x_{0:T}) \rightarrow p(x_{0:T}|z_{1:T})$ in distribution~\cite{blei2017variational}, we have $q^{\star}(x_{0:T}) \propto p(x_{0:T}|z_{1:T})$ asymptotically. Consequently, given the particle sets $\{x_t^i\}_{i=1}^{N_s}$ for all $t \in \{1, \ldots, T\}$, an approximation to the MAP trajectory $x^{\text{MAP}}_{0:T}$ in the second optimization in~\eqref{eq::MAP_est_formulation} is
\begin{align} \label{eq::Viterbi_cost1}
    \hat{x}^{\text{MAP}}_{0:T} = \arg \!\!\!\!\!\!\!\!\!\!\!\!\!\!\!\!\!\!\!\!\!\!\!\max_{x_{0:T} \in \{x_0^{\text{true}}\} \times \bigotimes_{k=1}^{T} \{x_k^{i}\}_{i=1}^{N_s}} \!\!\!\!\!\!\!\!\!\!\!\!\!\!p(x_{0:T}|z_{1:T}),
\end{align}
where $\bigotimes_{k=1}^{T}$ is the Cartesian product over time steps from $1$ to $T$. This implies that the trajectory $x_{1:T}$ is selected from the product of particle sets across all time steps. As long as the support of $q^{\star}(x_{0:T})$ includes the support of $p(x_{0:T} | z_{1:T})$, the estimate $\hat{x}^{\text{MAP}}_{0:T}$ converges asymptotically to $x_{0:T}^{\text{MAP}}$ as $N_s \rightarrow\infty$.

\begin{figure}[!t]
    \centering
    \begin{minipage}[t]{0.45\textwidth}
        \centering
        \includegraphics[width=\textwidth]{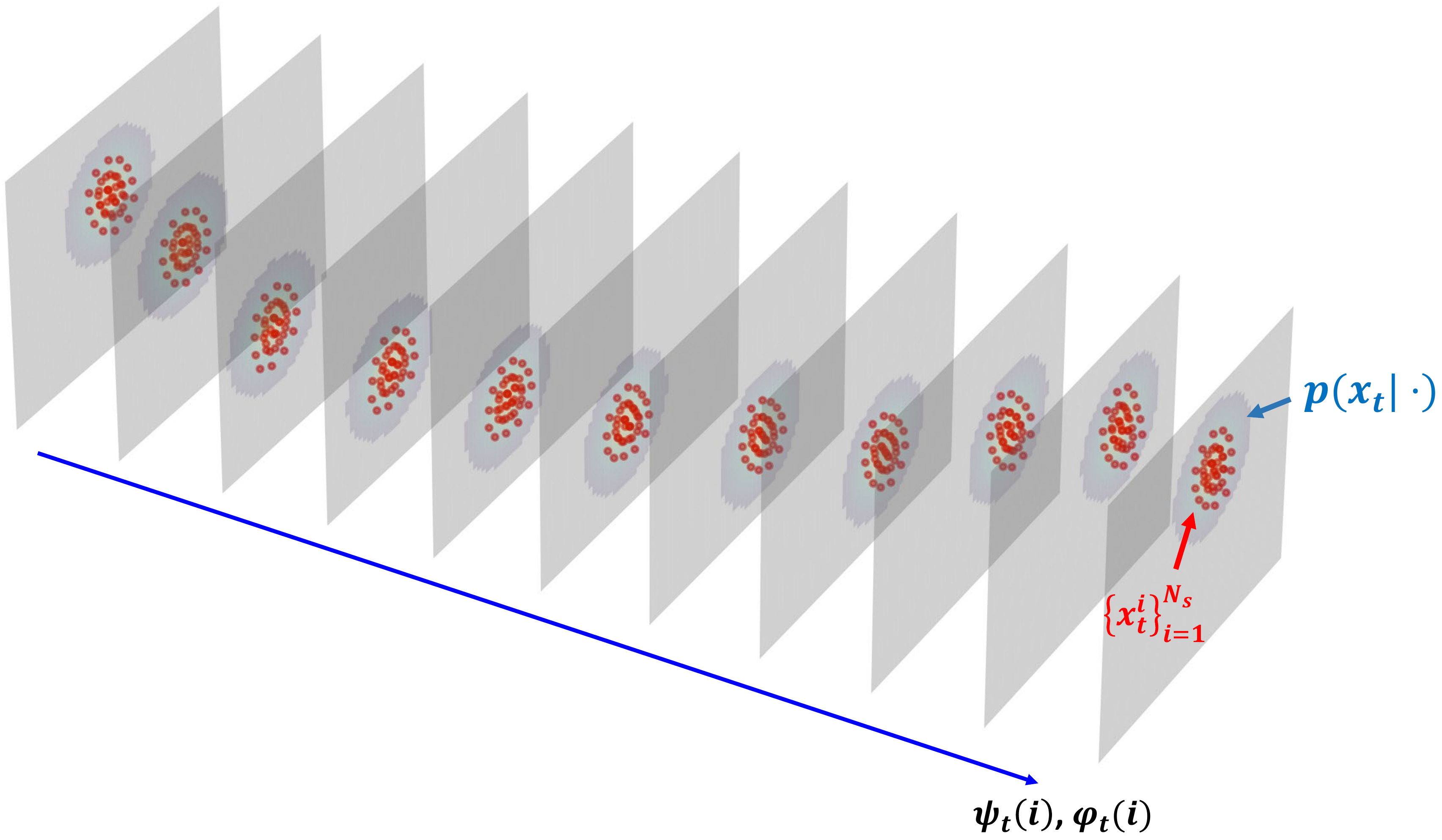}
        {{\scriptsize (a) Sequential particle update to obtain a finite discrete approximation of the state space.}}
        \label{fig::particle_DP1}
    \end{minipage}
    \hskip\baselineskip
    \begin{minipage}[t]{0.43\textwidth}
        \centering
        \includegraphics[width=\textwidth]{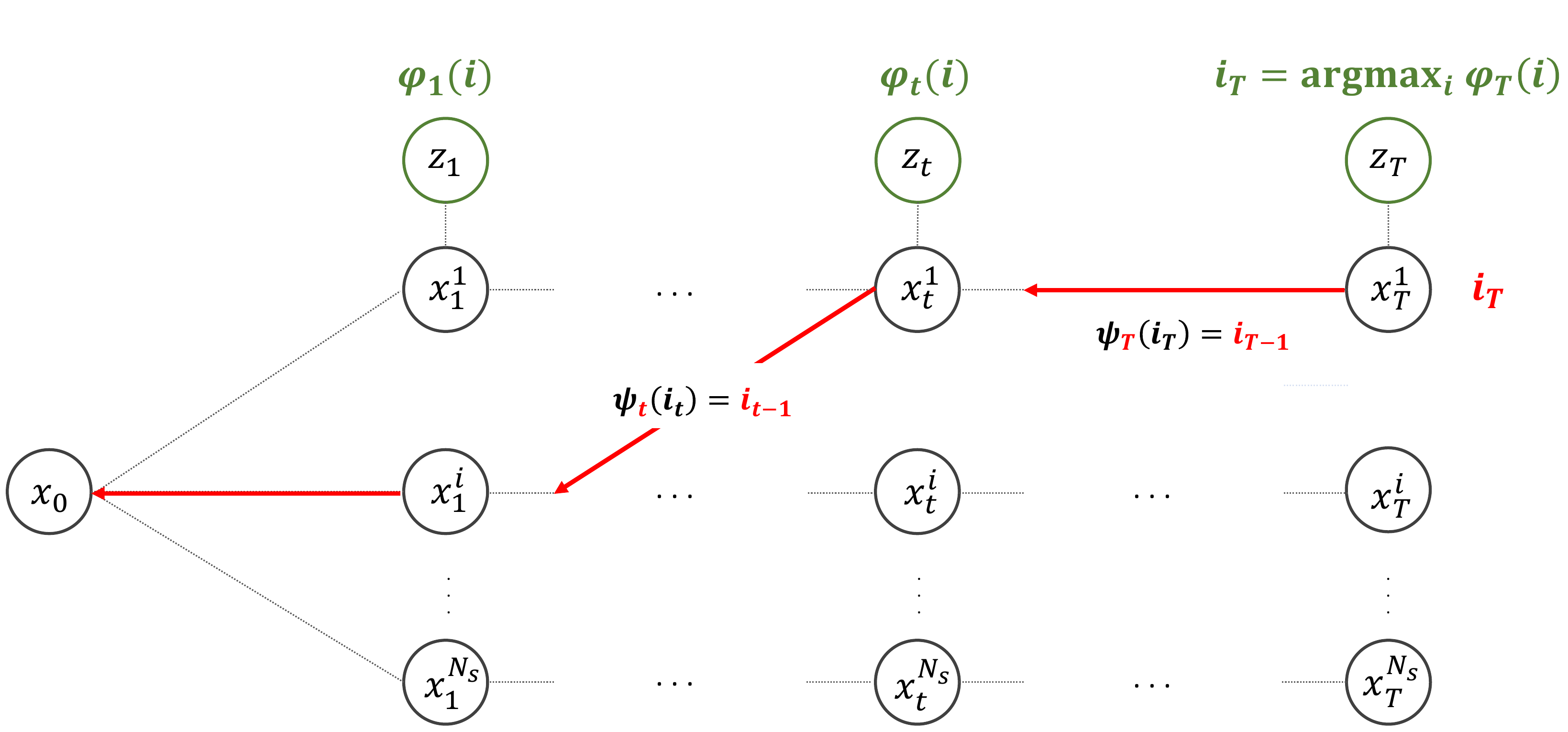}
        {{\scriptsize (b) Backtracking the MAP sequence via the memorization table $\psi(i)$.}}
        \label{fig::particle_DP2}
    \end{minipage}
    \caption{{\small Integration of SVGD-generated sequential particle sets with dynamic programming to yield globally optimal MAP trajectories.}}
    \label{fig::particle_DP}
\end{figure}

Taking the logarithm of the posterior in~\eqref{eq::Viterbi_cost1} and exploiting the first-order Markov structure of the model~\eqref{eq::SSM-x} and the conditional independence of observations (Assumption~\ref{assum::obs_indep}), we obtain, up to an additive constant:
\begin{align}
    \log p(x_{0:T}|z_{1:T}) \propto \log p(x_0) +
    \sum_{t=1}^T\nolimits \log p(x_t, z_t|x_{t-1}).
\end{align}

Therefore, the MAP trajectory approximation becomes:
\begin{align}
    \hat{x}^{\text{MAP}}_{0:T} = \arg \!\!\!\!\!\!\!\!\!\!\!\!\!\!\!\!\!\!\!\!\!\!\!\!\max_{x_{0:T} \in \{x_0^{\text{true}}\} \times \bigotimes_{k=1}^{T} \{x_k^{i}\}_{i=1}^{N_s}} \!\!\Bigl[\log p(x_0) + \!\sum_{t=1}^T  \log p(z_t,x_t|x_{t-1})\Bigr]\!.
\end{align}

\subsubsection{Viterbi-Style Dynamic Programming Algorithm}
The discrete optimization problem above can be efficiently solved using a Viterbi-style dynamic programming algorithm that leverages the temporal structure inherent in the factorized ELBO. The algorithm consists of two main phases: forward recursion for computing optimal accumulated scores, and backtracking for reconstructing the optimal sequence.

Finally, $\hat{x}^{\text{MAP}}_{0:T} = (\hat{x}^{\text{MAP}}_0, \hat{x}^{\text{MAP}}_1, \ldots, \hat{x}^{\text{MAP}}_T)$ is efficiently found by dynamic programming over this discretized particle set. The steps are as follows:

\textbf{Forward Recursion.} For $2 \leq t \leq T$ and for each particle $x_t^i$ (where $1 \leq i \leq N_s$), we compute $\psi_t(i)$ (the index of the optimal preceding particle) and $\varphi_t(i)$ (the maximum accumulated score to reach $x_t^i$) as:
\begin{subequations} \label{eq::forward_recursion}
\begin{align}
\psi_t(i)\! &= \arg\!\!\!\!\!\!\!\max_{j \in \{1,\dots,N_s\}} \Bigl\{\log p(x^i_t|x^j_{t-1}) + \varphi_{t-1}(j) \Bigl\}\!,\!\! \label{eq::forward_recursion1}\\
\varphi_t(i)\! &= \log p(z_t,x^i_t|x^{\psi_t(i)}_{t-1}) + \varphi_{t-1} \bigl( \psi_t(i) \bigl), \label{eq::forward_recursion2}
\end{align}
\end{subequations}
where $\varphi_1(i)$ is initialized based on $\log p(z_1,x_1^i|x_0^{\text{true}})$. This recursion ensures that $\varphi_t(i)$ stores the maximum score for any path ending at particle $x_t^i$, by considering all possible preceding particles $x_{t-1}^j$ and their optimal accumulated scores $\varphi_{t-1}(\psi_t(i))$.

\textbf{Backtracking the MAP Sequence.} After computing $\varphi_T(i)$ for all $i$, the index of the optimal final particle is found as $i_T = \arg\!\max_{i} \varphi_T(i)$. The complete MAP sequence $\hat{x}^{\text{MAP}}_{0:T}$ is then reconstructed by tracing back through the memorization table $\psi_t(i)$ from $t = T-1$ down to $1$:
\begin{align}
\hat{x}^{\text{MAP}}_t = x_{t}^{i_t}, \quad \text{where}\quad i_t = \psi_{t+1}(i_{t+1}), \nonumber
\end{align}
with $\hat{x}^{\text{MAP}}_0 = x_0^{\text{true}}$. This backtracking step ensures that the reconstructed path corresponds to the sequence of particles that yielded the maximum accumulated score, thereby identifying the MAP sequence estimator.

This two-stage approach offers a robust and computationally efficient solution for MAP sequence estimation in multimodal settings. First, SVGD flexibly captures complex posterior structures with a particle (support) set; then, particle-based dynamic programming decodes the optimal trajectory from these samples.

\subsubsection{Theoretical Guarantees}
The correctness of our particle-based dynamic programming approach is established through the connection between ELBO maximization and MAP trajectory recovery in the small-$\varepsilon$ regime, as formalized in the following theorem.

In summary, \textsf{Stein-MAP-Seq} computes the MAP sequence using a compact set of $N_s$ particles at each time step $t$. The optimal sequence $\hat{x}_{0:T}^{\mathrm{MAP}}$ is obtained by backtracking through the memorization table $\psi_t(i)$, which is constructed from the accumulated scores $\varphi_t(i)$ in Algorithm~\ref{alg:SteinMAPDP}. In the small-$\varepsilon$ regime, $\hat{x}_{0:T}^{\mathrm{MAP}}$ coincides with the maximizer of the ELBO in Lemma~\ref{lem::factKL}, as stated in Theorem~\ref{thm::ELBO_MAP}. For notational simplicity in Theorem~\ref{thm::ELBO_MAP}, we let $s$ denote an arbitrary candidate trajectory from $\mathcal{S}$, with $s^\star$ denoting the maximizer, i.e., $s^\star = \hat{x}_{0:T}^{\mathrm{MAP}}$.

\begin{algorithm}[t]
\caption{\textsf{Stein-MAP-Seq}}
\label{alg:SteinMAPDP}
\SetAlgoLined
\LinesNumbered
\DontPrintSemicolon
\SetKwInOut{Input}{Input}\SetKwInOut{Output}{Output}
\Input{Transition probability $p(x_t|x_{t-1})$, initial state $x_0^{\text{true}}$, likelihood function $p(z_t|x_t)$, observations $z_{1:T}$, the number of particles $N_s$, the step size $\epsilon$.}
\BlankLine
\tcp{1) Initialization $(t=1)$}
\For{$i = 1,2,\dots,N_s$ \textbf{in parallel}}
    {
        $x_1^i \leftarrow  x_1^i + \epsilon \hat{\phi}^{\star}(x_1^i)$, \\
        where $\hat{\phi}^{\star}(x_1)$ from \eqref{eq::SVGD_initial_newupdate} given $x_0^{\text{true}}$
    }
\For{$i \leftarrow 1$ to $N_s$}
{$\varphi_1(i) = \log p(z_1,x_1^i|x_0^{\text{true}})$}
\tcp{2) Sequential Particle Update}
\For{$t=2,\dots,T$}{
    \For{$i = 1,2,\dots,N_s$ \textbf{in parallel}}
        {
            $x_t^i \leftarrow  x_t^i + \epsilon \hat{\phi}^{\star}(x_t^i)$, \\
            where $\hat{\phi}^{\star}(x_t)$ from \eqref{eq::SVGD_newupdate} given $\{x_{t-1}^j\}_{j=1}^{N_s}$
        }
    \tcp{3) Forward Recursion}
    \For{$i \leftarrow 1$ to $N_s$}{$\psi_t(i)\leftarrow\eqref{eq::forward_recursion1}$ \\
    $\varphi_t(i)\leftarrow\eqref{eq::forward_recursion2}$
    }
}
\tcp{4) Backtracking}
$i_T = \arg\!\max_{i} \varphi_T(i)$ \\
$\hat{x}^{\text{MAP}}_T = x^{i_T}_T$ \\
\For{$t=T-1,\dots,1$}{
    $i_t = \psi_{t+1}(i_{t+1})$\\
    $\hat{x}^{\text{MAP}}_t = x_{t}^{i_t}$   
}
$\hat{x}^{\text{MAP}}_0 = x_0^{\text{true}}$ \\ 
\Return{$\hat{x}^{\text{MAP}}_{0:T} = \{\hat{x}^{\text{MAP}}_0,\hat{x}^{\text{MAP}}_1,\dots,\hat{x}^{\text{MAP}}_T\}$}
\end{algorithm}

\begin{thm}\label{thm::ELBO_MAP}
{(Equivalence of ELBO Maximization and the MAP Trajectory over Particle Sets in the Small-$\varepsilon$ Regime). Let Assumption~\ref{assump:smoothness} and~\ref{assump::MAP} hold, and let $\mathcal{S} \subseteq \{x_0^{i}\}_{i=1}^{N_s} \times \cdots \times \{x_T^{i}\}_{i=1}^{N_s}$ denote the set of admissible trajectories over the particle sets. For $s=(s_0,\ldots,s_T) \in \mathcal{S}$, define the trajectory score as
\begin{align}
    J(s) = \log p(s_0) + \sum_{t=1}^T\nolimits \log p(s_t, z_t|s_{t-1}). \nonumber
\end{align}
For $\varepsilon>0$, define $q_{\varepsilon}(x_0)=\mathcal{N}(x_0;s_0,\varepsilon \mathbf{I})$, $q_{\varepsilon}(x_t|x_{t-1})=\mathcal{N}(x_t;s_t,\varepsilon \mathbf{I})$ for all $t=1,\dots,T$, and, from~\eqref{eq::ELBO_factorized} in Lemma~\ref{lem::factKL}, let
\begin{align}
    \mathcal{L}_{\varepsilon}(s) := \mathcal{L}_{\varepsilon}(q_{\varepsilon}) + (T+1)H_{\varepsilon}, \nonumber
\end{align}
where 
\begin{align}
    \mathcal{L}_{\varepsilon}(q_{\varepsilon}) &\!=\! \mathbb{E}_{q_{\varepsilon}(x_0)}[\log p(x_0)] \!+\! \mathbb{E}_{q_{\varepsilon}(x_0)}\mathbb{E}_{q_{\varepsilon}(x_1|x_0)}[\log p(z_1,x_1|x_0)] \nonumber \\
    &+ \sum_{t=2}^T\nolimits \mathbb{E}_{q_{\varepsilon}(x_{t-1}|x_{t-2})}\mathbb{E}_{q_{\varepsilon}(x_{t}|x_{t-1})}[\log p(z_t,x_t|x_{t-1})], \nonumber    
\end{align}
and $H_{\varepsilon}$ is the entropy of a Gaussian, which does not depend on $s$ and is therefore constant. Then, for every $s \in \mathcal{S}$,
\begin{align}
    \lim_{\varepsilon \rightarrow 0}\mathcal{L}_{\varepsilon}(s) = J(s). \nonumber
\end{align}

Moreover, $\mathcal{S}$ is finite, and suppose that there exists a unique $s^{\star} \in \arg\max_{s\in \mathcal{S}}J(s)$ with gap $\delta := J(s^{\star}) - \max_{s \neq s^{\star}} J(s) > 0$. Then there exists $\varepsilon_0 > 0$ such that for all $0<\varepsilon<\varepsilon_0$,
\begin{align}
    \mathcal{L}_{\varepsilon}(s^{\star})>\mathcal{L}_{\varepsilon}(s) ~~\text{for all}~~s\neq s^{\star}. \nonumber
\end{align}
}
\end{thm}
\medskip
\begin{proof}
    See Appendix~\ref{sec::appen_E}.
\end{proof}

Theorem~\ref{thm::ELBO_MAP}, combined with Lemma~\ref{lem::conditional_svgd}, establishes rigorous convergence guarantees for \textsf{Stein-MAP-Seq}. Lemma~\ref{lem::conditional_svgd} ensures that the sequential SVGD updates provably decrease the marginalized conditional KL divergence at each time step, guaranteeing that the generated particles $\{x_t^{i}\}_{i=1}^{N_s}$ concentrate around the modes of the target conditional posteriors $p(z_t, x_t|x_{t-1})$. Theorem~\ref{thm::ELBO_MAP} then establishes that the Viterbi-style dynamic programming over these particle sets recovers the trajectory that maximizes the ELBO, which converges to the true log-posterior $J(s)$ as $\varepsilon \to 0$. Together, these results provide end-to-end theoretical justification: the SVGD-generated particles provably approximate optimal conditional distributions (Lemma~\ref{lem::conditional_svgd}), and the dynamic programming algorithm provably identifies the globally optimal MAP trajectory within this discrete particle support (Theorem~\ref{thm::ELBO_MAP}), with approximation error $|\mathcal{L}_\varepsilon(s) - J(s)| \to 0$ as particle quality improves.

Having established the theoretical foundations and convergence guarantees, we now analyze the computational complexity of \textsf{Stein-MAP-Seq}.

\subsection{Computational Analysis of \textsf{Stein-MAP-Seq}}\label{subsec::Computation}

Table~\ref{Table_Comp} summarizes the computational complexity of the main components of \textsf{Stein-MAP-Seq}. The dominant cost in the SVGD update arises from the kernel interaction term, which scales quadratically with the number of particles $N_s$ and linearly with the state dimension $n_x$, resulting in an $\mathcal{O}(N_s^2 n_x)$ complexity. Although this double summation in Eq.~\eqref{eq::SVGD_newupdate} can be computationally demanding, it is consistent with standard SVGD-based methods and, importantly, is fully parallelizable. Both the log-posterior gradient evaluations and the pairwise kernel interactions are independent across particles, making them well-suited for efficient GPU acceleration. To complement this theoretical analysis, Appendix~\ref{sec::appen_F} reports GPU-based per-step runtime evaluations of~\eqref{eq::SVGD_newupdate} as a function of the number of particles and state dimension (Fig.~\ref{fig:appendix:com}), demonstrating stable and predictable runtime behavior in practice.

Beyond parallelization, \textsf{Stein-MAP-Seq} further mitigates computational cost through the mode-seeking yet dispersive behavior of SVGD. By concentrating particles around multiple dominant modes while maintaining separation between them, SVGD produces a compact set of high-quality candidates for MAP trajectory estimation. As a result, \textsf{Stein-MAP-Seq} operates reliably with a significantly smaller number of particles than MC sampling–based MAP-Seq methods.

The primary computational bottleneck in MAP sequence estimation arises from the forward recursion in~\eqref{eq::forward_recursion}, which incurs an $\mathcal{O}(N_s^2)$ cost per time step. Crucially, the reduction in the effective particle count $N_s$ directly translates into a substantial reduction in this dominant cost. Moreover, the forward recursion remains amenable to parallelization through independent score computations followed by max-reduction operations, and its complexity is independent of the state dimension $n_x$. The backtracking step scales linearly with the trajectory length $T$ and contributes negligible overhead. Consequently, \textsf{Stein-MAP-Seq} achieves computationally efficient MAP sequence estimation by combining parallelizable SVGD updates with particle-efficient mode discovery, effectively reducing the cost of the dominant forward recursion even in high-dimensional systems. This distinction explains why the practical efficiency of \textsf{Stein-MAP-Seq} stems not from asymptotic complexity reduction, but from particle efficiency and parallel execution.

Overall, the above analysis highlights that the computational efficiency of \textsf{Stein-MAP-Seq} is closely tied to its parallel-friendly structure. As a result, \textsf{Stein-MAP-Seq} is naturally suited for efficient execution on parallel computing architectures. Although SVGD requires synchronization to aggregate kernel interactions, this synchronization is limited to standard reduction operations and does not involve resampling or particle selection. In contrast, Monte Carlo sampling methods such as particle filters~\cite{gustafsson2010particle} rely on resampling procedures that introduce non-deterministic state selection, resulting in variable computation paths across runs. Moreover, particles cannot be processed independently because resampling depends on global statistics, such as normalized weights and cumulative sums, which significantly limit parallel efficiency in practice.

\begin{table}[t]
\centering
\caption{Computational complexity of \textsf{Stein-MAP-Seq} components}
\label{Table_Comp}
\small
\resizebox{0.48\textwidth}{!}{\begin{tabular}{ccc}
    \hline
    \textbf{Component} & \textbf{Time Complexity} & \textbf{Parallelizable\footnotemark} \\
    \hline
    Log-posterior gradient (in Eq.~\eqref{eq::SVGD_newupdate}) & $\mathcal{O}(N_s n_x)$ & Yes (particle-wise independent)\\
    Kernel interaction (in Eq.~\eqref{eq::SVGD_newupdate}) & $\mathcal{O}(N_s^2 n_x)$ & Yes (pairwise independent; reduction) \\
    Forward recursion (Eq.~\eqref{eq::forward_recursion}) & $\mathcal{O}(N_s^2)$ & Yes (per time step; max-reduction) \\
    Backtracking & $\mathcal{O}(T)$ & No (time-serial) \\
    \hline
    \end{tabular}}
\end{table}
\footnotetext{The \emph{parallelizable} column indicates execution structure and does not change the asymptotic computational complexity.}


\section{Demonstrations}\label{sec::results}

In this section, we rigorously evaluated the proposed method on challenging multimodal scenarios, including (A) nonlinear dynamics with \emph{ambiguous} measurements, (B) pose estimation under \emph{unknown} data associations with \emph{outliers}, (C) range-only (wireless) localization under temporary \emph{unobservability}, and (D) high-dimensional manipulator with \emph{kinematic redundancy}.

We compared the performance of several state‑estimation methods:
\begin{itemize}
  \item \textbf{Extended Kalman Filter (EKF)}~\cite{gelb1974applied} and \textbf{Particle Filter (PF)}~\cite{doucet2001introduction}, both MMSE estimators, and \textbf{Extended Kalman Smoother (EKS)}, a smoothing-based MMSE sequence estimator;
  \item \textbf{Iterated EKF (iEKF)}, a local MAP estimator\footnotemark{}\footnotetext{The iterative procedure can be interpreted as the Gauss-Newton method for finding the MAP estimate at the correction step~\cite{sarkka2023bayesian}.}, and \textbf{Iterated EKS (iEKS)}, a local MAP sequence estimator;
  \item \textbf{GTSAM}~\cite{gtsam}, batch MAP sequence estimator via factor-graph optimization;
  \item \textbf{PF‑based MAP estimate (PF‑MAP)}~\cite{driessen2008particle} and \textbf{PF‑based MAP‑sequence estimate (PF‑MAP‑Seq)}~\cite{godsill2001maximum};
  \item \textbf{Stein Particle Filter (SPF)}~\cite{maken2022stein} and its MAP variant \textbf{SPF‑MAP}, which applies the PF‑MAP procedure to the SPF output.
\end{itemize}

All methods were implemented in Python and evaluated on a laptop equipped with an Intel Core i7-10510U CPU and 32 GB of RAM. Our method was implemented in PyTorch and executed on a GPU to accelerate the parallel particle updates in~\eqref{eq::SVGD_newupdate}. This GPU acceleration affects only computational speed and does not modify the algorithmic structure or estimation results. For PF-based methods, stratified resampling was applied at each time step. For the Stein-based methods, the step size was set to $\epsilon = 0.005$, and the number of iterations was fixed to $100$ across all scenarios. We employ a radial basis function (RBF) kernel to model smooth interactions among particles. For GTSAM, we adopt the Levenberg–Marquardt optimizer with a maximum of 50 iterations to improve robustness under strong nonlinearities.

\begin{figure}[t]
    \centering
    \includegraphics[width=0.92\linewidth]{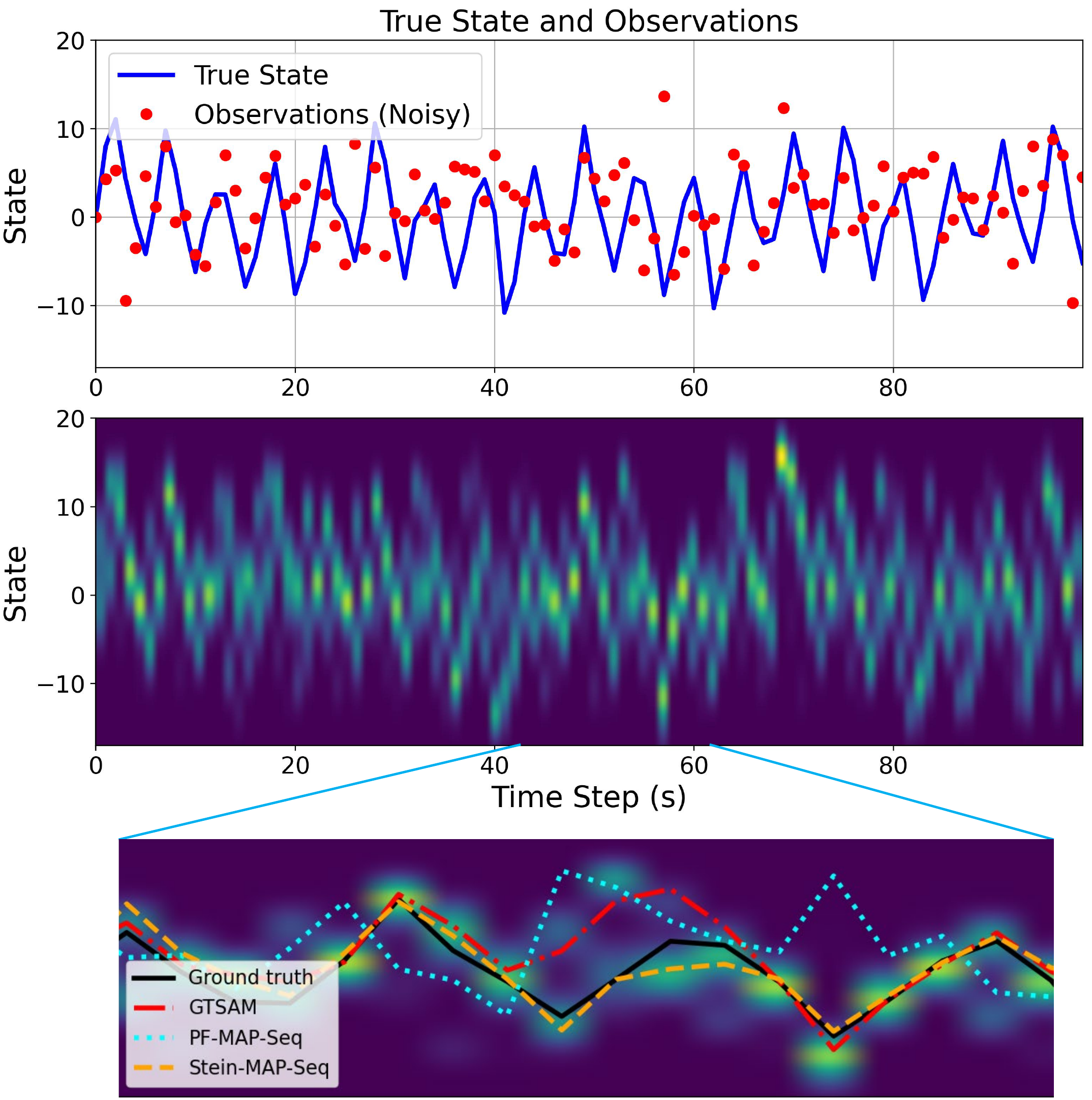}
    \caption{\textbf{(Top)} True state trajectory and noisy observations (red dots) with sign ambiguity, resulting in a bimodal posterior distribution. \textbf{(Middle)} Time-indexed grid-based posterior density illustrating the evolution of bimodality. \textbf{(Bottom)} One simulation over the interval ($43$–$62\mathrm{s}$) showing that \textsf{Stein-MAP-Seq} resolves bimodality and recovers the correct trajectory, whereas baseline MAP-Seq estimators exhibit incorrect mode selection.}
    \label{fig:Scenario1}
\end{figure}

\subsection{1D Nonlinear Dynamics with Ambiguous Measurement}
In the first example, we consider a benchmark problem from~\cite{garcia2016iterated}, where the system model~\eqref{eq::SSM} is given by\footnotemark{}\footnotetext{The models correspond to the form $x_t \sim p(x_t|x_{t-1}) = \mathcal{N}(f(x_{t-1}),\, v_{t-1})$ and $z_t \sim p(z_t|x_t) = \mathcal{N}(h(x_t),\, r_t).$}
\begin{subequations}
\begin{align}
    x_t &= 0.9 x_{t-1} \!+\! \frac{10x_{t-1}}{1+x_{t-1}^2} \!+ 8\cos{\bigl(1.2(t\!-\!1)\bigl)} + v_{t-1},\! \label{eq::sim_m1} \\ \label{eq::sim_m2} 
    z_t &= 0.05 x_t^2 + r_t, 
\end{align}
\end{subequations}
with $v_{t-1}\sim\mathcal{N}(0, 5)$, $r_t\sim\mathcal{N}(0, 4^2)$, and $x_0\sim\mathcal{N}(0, 5)$. The significance of this example is that the term $x_t^2$ in~\eqref{eq::sim_m2} introduces a sign ambiguity, e.g., $x_t=+a$ and $x_t=-a$ yield the same measurement, resulting in a bimodal posterior (see Fig.~\ref{fig:Scenario1}, middle). Moreover, the large process uncertainty spreads the predicted prior across competing modes, while the high measurement uncertainty weakens observational correction, preventing reliable local mode discrimination.

To evaluate the state estimation methods, we generated $50$ independent trajectories of $100$ time steps each. For each estimator, we computed the root‑mean‑square error (RMSE) of the state estimates, averaged over all Monte Carlo runs. The results in Fig.~\ref{fig:Scenario1}, bottom, Fig.~\ref{fig::Scenario1-2}, and Table~\ref{Table_Ex1} demonstrate that \textsf{Stein-MAP-Seq} consistently outperformed competing estimators, highlighting its robustness to bimodal posterior distributions under high process and measurement uncertainty. This performance advantage arises from the complementary interaction between SVGD’s mode-seeking behavior and Viterbi-style dynamic programming. SVGD concentrates particles around dominant posterior modes using only a small particle set, while Viterbi's forward recursion exploits this focused representation to recover the globally optimal trajectory over the sampled space.

\begin{figure}[t]
    \centering
    \includegraphics[width=0.95\linewidth]{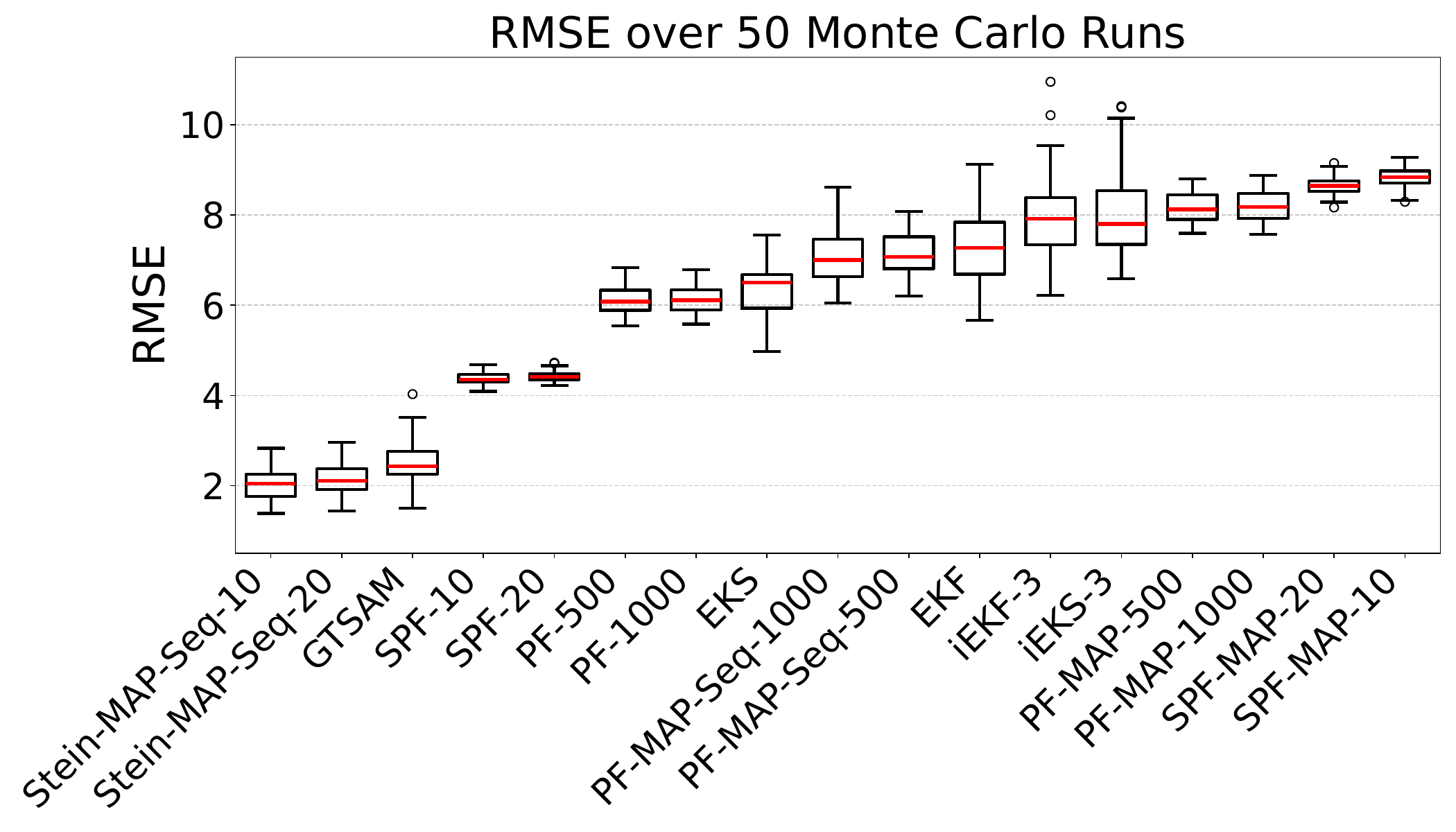}
    \caption{RMSE results over $50$ Monte Carlo simulations with $100$ time steps are reported for a bimodal posterior under high process and measurement uncertainty. Finite-state estimators are generally more robust than Gaussian-assumed estimators in multimodal settings, except for point-wise MAP estimators. Point-wise MAP estimators frequently select incorrect modes and therefore perform worse than MMSE estimators, which effectively average across modes. Among MAP sequence estimators, \textsf{Stein-MAP-Seq} and GTSAM achieve the best performance, whereas PF-based MAP methods exhibit the highest RMSE due to unstable mode selection.}
    \label{fig::Scenario1-2}
\end{figure}

The methods, such as EKF, EKS, iEKF-3, and iEKS-3, rely on Gaussian (unimodal) assumptions and local linearization, making them prone to converge to locally optimal trajectories. In this setup, large process noise spreads the prior across competing modes, while high measurement uncertainty weakens observational correction, preventing both filtering and smoothing from resolving the sign ambiguity. iEKF-3 and iEKS-3 were initialized with EKF and iEKF-3 estimates, respectively, and refined the trajectory through iterative forward–backward local linearization, where the number following the method name denotes the iteration count. However, because the initial trajectories were already biased toward incorrect modes, the subsequent local optimization in iEKF-3 and iEKS-3 reinforced these errors, resulting in worse performance than EKF and EKS, as shown in Table~\ref{Table_Ex1}.

\begin{table}[t]
\begin{center}
\caption{Robustness Evaluation for Scenario (A), Averaged RMSE}
\label{Table_Ex1}
\scriptsize 
\resizebox{0.45\textwidth}{!}{
    \begin{tabular}{c c c|c c c c} 
    \hline
    \multicolumn{3}{c|}{\textbf{Gaussian-assumed Estimation}} 
    & \multicolumn{4}{c}{\textbf{Finite-state Estimation}} \\
    \cline{1-3}\cline{4-7}
    EKF & EKS & GTSAM & $N_s$ & PF & PF-MAP & PF-MAP-Seq \\
    \hline
    &   && 5e2 & $6.1288$ & $8.1654$ & $7.0949$ \\
    $7.2289$ & $6.3422$ & $2.5564$ & 1e3 & $6.1320$ & $8.1829$ & $7.0696$ \\
    &   && 2e3 & $\mathbf{6.1281}$ & $8.1863$ & $7.1363$ \\
    \hline
    iEKF-3 & iEKS-3 & GTSAM-Stein & $N_s$ & SPF & SPF-MAP & \textsf{Stein-MAP-Seq} \\
    \hline
    & && 1e1 & $4.3719$ & $8.8307$ & $\mathbf{2.0462}$ \\
    $7.9307$ & $8.1320$ & $\mathbf{1.5509}$ & 2e1 & $4.4282$ & $8.6583$ & $2.1447$\\
    & && 4e1 & $4.4964$ & $8.8143$ & $2.2761$\\
    \hline
    \end{tabular}}
\end{center}
\normalsize
\end{table}

GTSAM performed batch optimization over the entire trajectory. However, under high process and measurement uncertainty, the resulting factor graph became weakly constrained, making the optimization highly sensitive to initialization. As a result, GTSAM frequently converged to an incorrect mode (see Fig.~\ref{fig:Scenario1}, bottom), yielding suboptimal trajectory estimates. When initialized with the trajectory estimate obtained from \textsf{Stein-MAP-Seq} (see Table~\ref{Table_Ex1}), GTSAM-Stein consistently converged to the correct mode, demonstrating the effectiveness of \textsf{Stein-MAP-Seq} as an initialization strategy for nonconvex batch MAP optimization.

For PF-MAP-Seq, even with a large number of particles ($\geq 500$) drawn via random sampling, its mode-seeking capability was markedly weaker than that of \textsf{Stein-MAP-Seq}, which operated with significantly fewer particles ($\leq 50$). This behavior arose from particle path degeneracy induced by resampling and sampling variability under a bimodal posterior, where importance sampling failed to concentrate particles around a single dominant trajectory consistently. As a result, MAP decoding became unstable and prone to mode switching, producing zig-zag trajectories (see Fig.~\ref{fig:Scenario1}, bottom). In contrast, SVGD promoted coherent particle concentration around dominant modes and enabled reliable MAP sequence recovery. For PF-MAP and SPF-MAP, the estimators converged to incorrect modes due to the absence of explicit temporal dependency modeling, as they operated on marginal state estimates rather than entire trajectories. Consequently, their estimation performance was the worst among all methods, even inferior to MMSE-based estimators such as EKF, PF, and SPF.

\begin{figure}[t]
    \centering
    \includegraphics[width=0.95\linewidth]{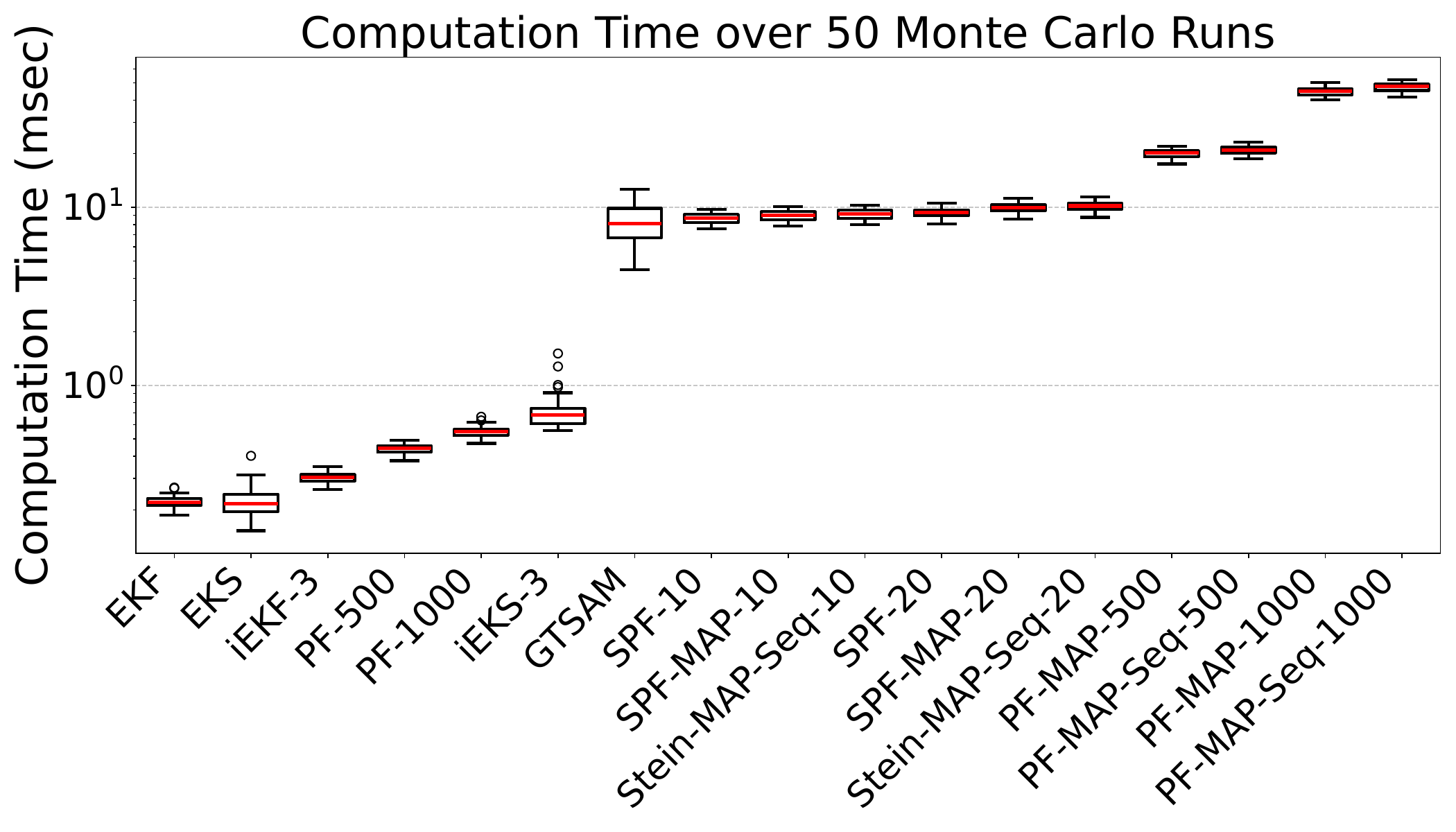}
    \caption{The averaged computation time per step is reported over $50$ Monte Carlo simulations. Filtering-based methods are more computationally efficient than MAP-Seq estimators, since MAP-Seq estimation requires trajectory-level inference. PF-MAP-Seq estimators incur the highest computational cost due to the large number of particles required, whereas \textsf{Stein-MAP-Seq} achieves competitive performance with fewer particles and computational cost comparable to GTSAM.}
    \label{fig::Scenario1-3}
\end{figure}

Among the Stein variants, \textsf{Stein-MAP-Seq} exhibited superior performance compared to SPF and SPF-MAP, despite their explicit modeling of temporal dependencies across the trajectory. Notably, \textsf{Stein-MAP-Seq} achieved higher estimation accuracy using significantly fewer particles. This advantage arises because, unlike MMSE-oriented estimators that require accurate approximation of the full posterior distribution, MAP-Seq estimation benefited from the mode-seeking behavior and the repulsive interactions in~\eqref{eq::SVGD_newupdate}, enabling reliable forward recursion and, consequently, accurate MAP trajectory recovery.

As shown in Fig.~\ref{fig::Scenario1-3}, filtering-based methods incurred the lowest computation time, whereas MAP-Seq estimators were inherently more expensive due to trajectory hypothesis management. Among finite-state estimators, the Stein series maintained a moderate computational cost with online forward recursion that scaled favorably with the number of particles, remaining comparable to GTSAM in practice while significantly outperforming PF-based MAP estimators. In contrast, PF-MAP-Seq exhibited rapidly increasing computation time as the particle count grew, reflecting the high computational burden of forward recursion over a large set of MAP candidate trajectories. Overall, these results indicated that \textsf{Stein-MAP-Seq} achieved a favorable accuracy–computation trade-off with a reduced number of particles required for reliable MAP-Seq estimation.

\subsection{Pose Estimation under Unknown Data Association} \label{sec::pos_est}

In the second example, we consider pose estimation on SE(2) under unknown data association with outliers. The discrete‐time pose is updated by
\begin{align}
    \begin{bmatrix}
        x_t \\
        y_t \\
        \theta_t
    \end{bmatrix}
    \!\!\!=\!\!\! 
    \begin{bmatrix}
    x_{t-1} + v_t\,\Delta t \,\cos\theta_t \\
    y_{t-1} + v_t\,\Delta t \,\sin\theta_t \\
    \theta_{t-1} + \omega_t\,\Delta t,
    \end{bmatrix},
\end{align}
with the pose process noise $\mathcal{N}(\mathbf{0},\mathbf{Q}_{t-1})$ where $\mathbf{Q}_{t-1}=\mathrm{diag}\bigl[(\alpha_x\,\Delta t)^2,\;(\alpha_y\,\Delta t)^2,\;(\alpha_\theta\,\Delta t)^2\bigr]$. The state vector is $\mathbf{x}_t = \bigl[x_t,\;y_t,\;\theta_t\bigr]^{\top}$ $(m,\,m,\,rad)$ with control inputs 
$\bigl[v_t,\;\omega_t\bigr]^{\top}$ $(m/s,\,rad/s)$, and $\Delta t$ denotes the sampling period. Here, $\alpha_x$ and $\alpha_y$ quantify translational noise per unit time along the body‑fixed $x$ and $y$ axes, and $\alpha_\theta$ quantifies rotational noise per unit time. In our simulation, we use $\alpha_x = \alpha_y = 0.3$, $\alpha_\theta = 0.105$, $\Delta t = 0.1$, and initialize the covariance as 
$\Sigma_0 = \mathrm{diag}(0.2^2,\,0.2^2,\,0.035^2)$.

The range and bearing measurements to landmark $l$ are given by:
\begin{subequations}
\begin{align} \label{eq::Obs_Prob1}
    z^r_{t,l} &= \bigl\lVert \Gamma_l - P_t \bigr\rVert_{2} + r^r_{t,l}, \\
    z^b_{t,l} &= \operatorname{atan2}\bigl(y_t - \Gamma_{l,y},\,x_t - \Gamma_{l,x}\bigr) + r^b_{t,l},
\end{align}
\end{subequations}
where $z^r_{t,l} \in \mathbb{R}_{\geq 0}\,(m)$, $z^b_{t,l} \in [-\pi, \pi]\,(rad)$, $\Gamma_l = [\Gamma_{l,x}, \Gamma_{l,y}]^{\top} \in \real^2$ is the $l$-th anchor position, and $P_t = [x_t,\,y_t]^\top \in \real^2$ is the robot’s position at time $t$. Here, $\|\cdot\|_2$ denotes the Euclidean norm, and the measurement noises are $r^r_{t,l}\sim \mathcal{N}(0, 1^2)$ and $r^b_{t,l}\sim \mathcal{N}(0, 0.44^2)$.

In this scenario, the robot receives observations from four landmarks with unknown data associations. For a given measurement $\mathbf{z}_t = [\,z^r_{t},\,z^b_{t}\,]^\top$, the identity of the generating landmark $l$ is unknown. To address this, we marginalize over $l$, yielding the association‑marginalized (data-association-free) likelihood below:
\begin{align}
    p(\mathbf{z}_t \mid \mathbf{x}_t)
    = \sum_{l=1}^{4}\nolimits p\bigl(\mathbf{z}_t \mid \mathbf{x}_t, l\bigr)\,p(l)\,,  
\end{align}
where $p(l)=1/4$. This formulation results in an equally weighted mixture of the four single-landmark likelihoods, producing a multimodal posterior distribution with multiple peaks, analogous to a Gaussian mixture model. 

\begin{figure*}[!t]
    \centering
    \begin{minipage}[t]{0.22\textwidth}
        \centering
        \includegraphics[width=\textwidth]{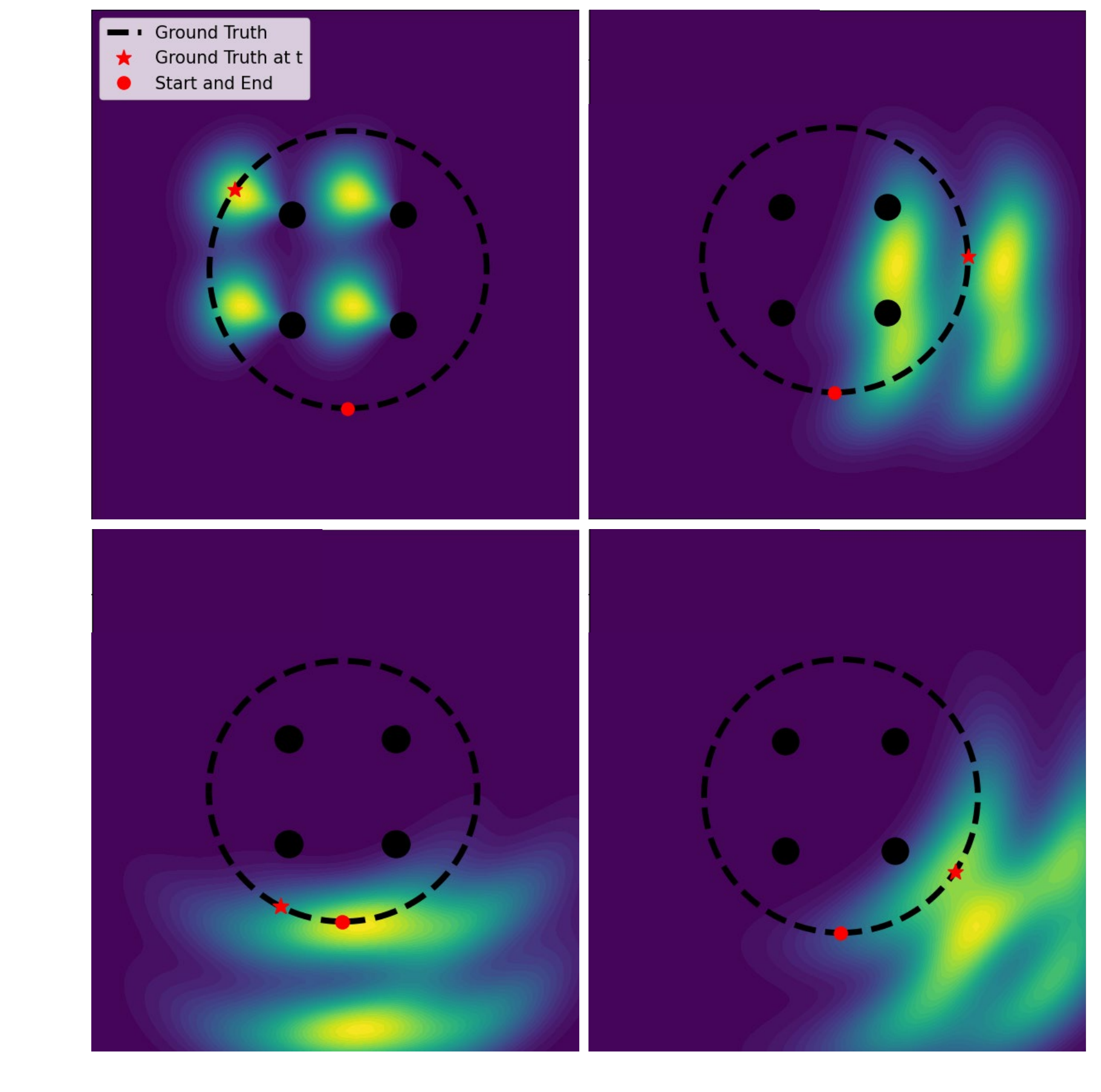}
        {{\scriptsize (a) Multimodal Likelihoods}}
    \end{minipage}
    \hskip\baselineskip
    \begin{minipage}[t]{0.22\textwidth}
        \centering
        \includegraphics[width=\textwidth]{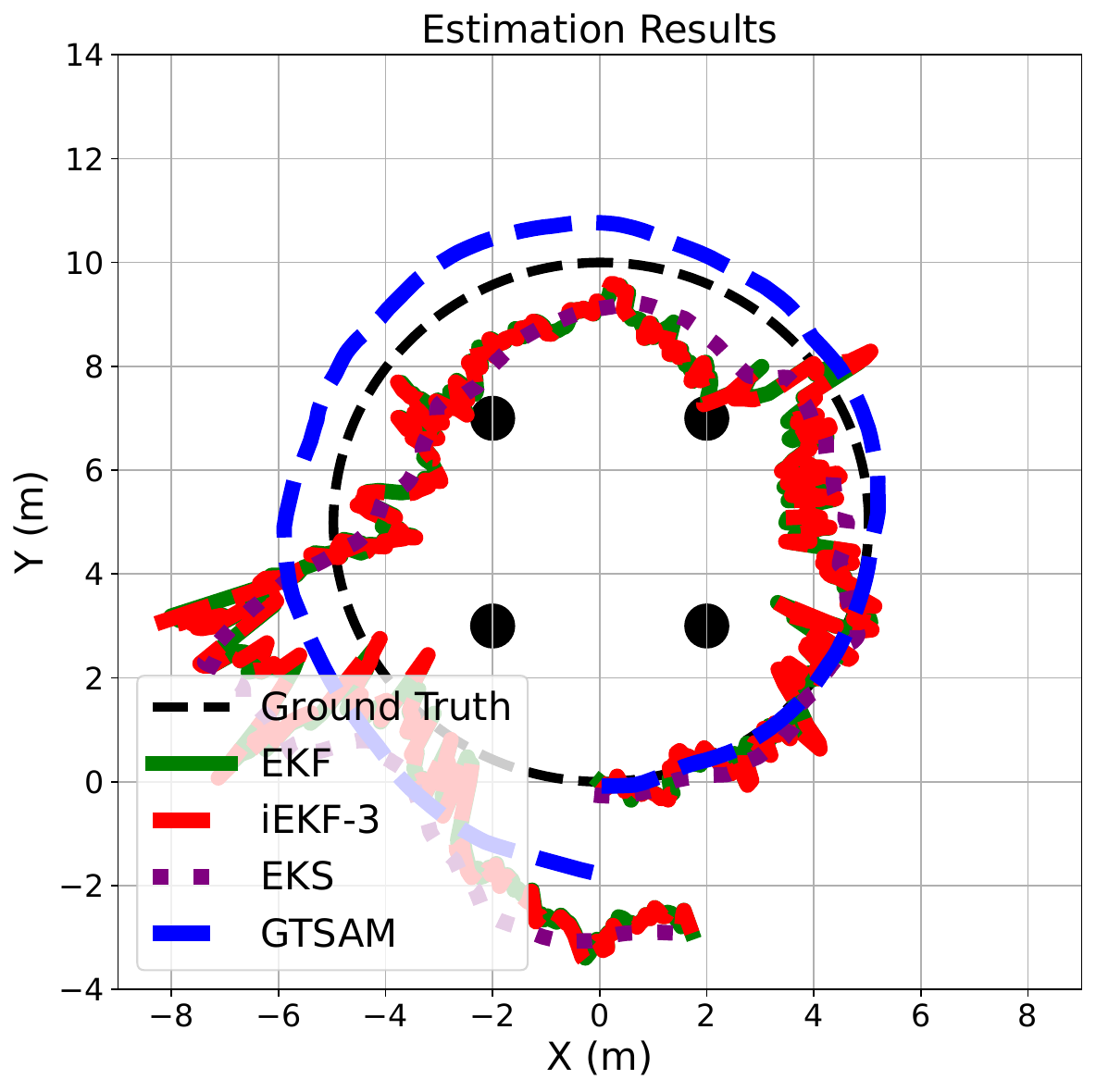}
        {{\scriptsize (b) Gaussian-assumed Estimation}}
    \end{minipage}
    \hskip\baselineskip
    \begin{minipage}[t]{0.22\textwidth}
        \centering
        \includegraphics[width=\textwidth]{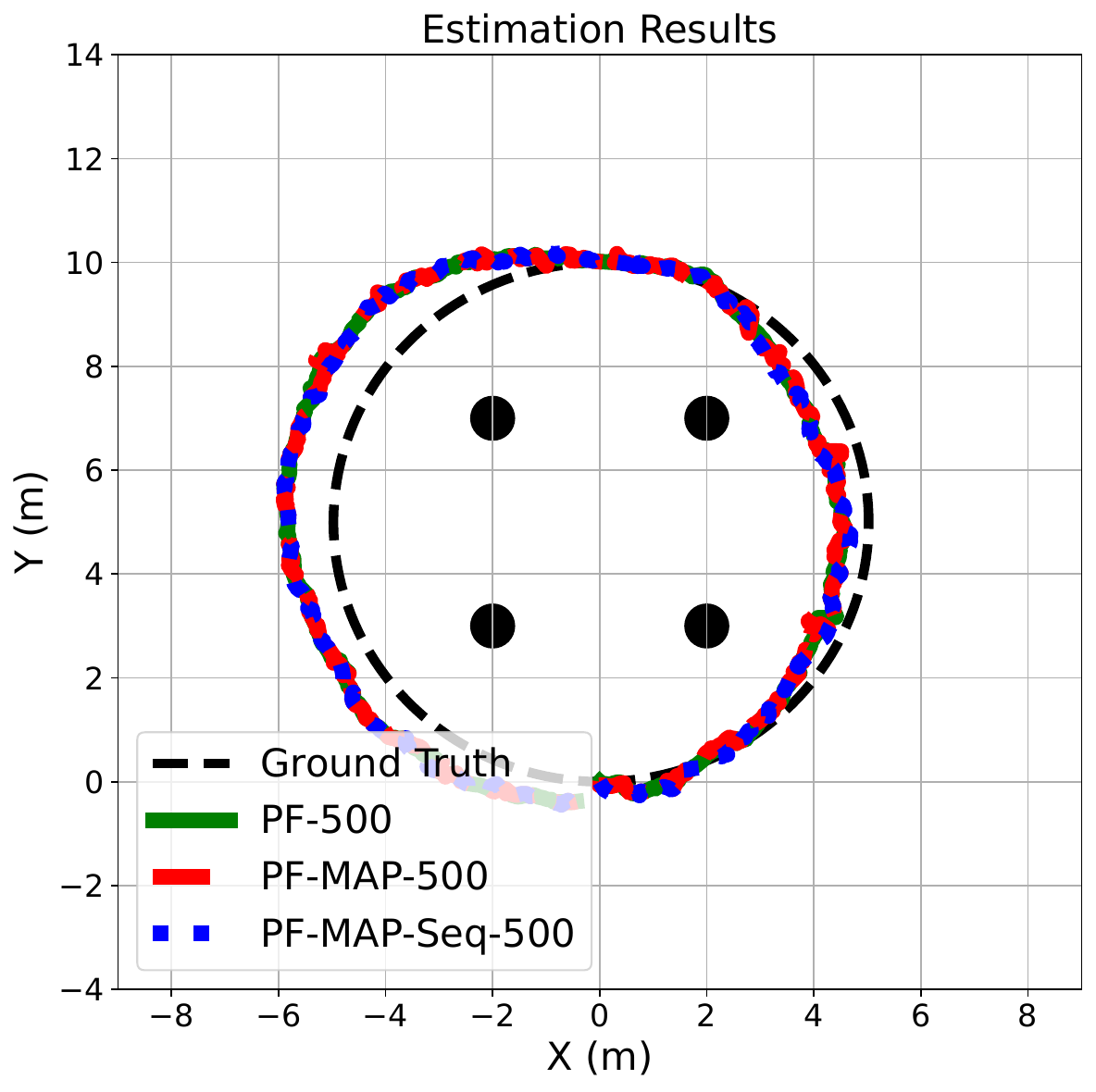}
        {{\scriptsize (c) PF-based Estimation}}
    \end{minipage}
    \hskip\baselineskip
    \begin{minipage}[t]{0.22\textwidth}
        \centering
        \includegraphics[width=\textwidth]{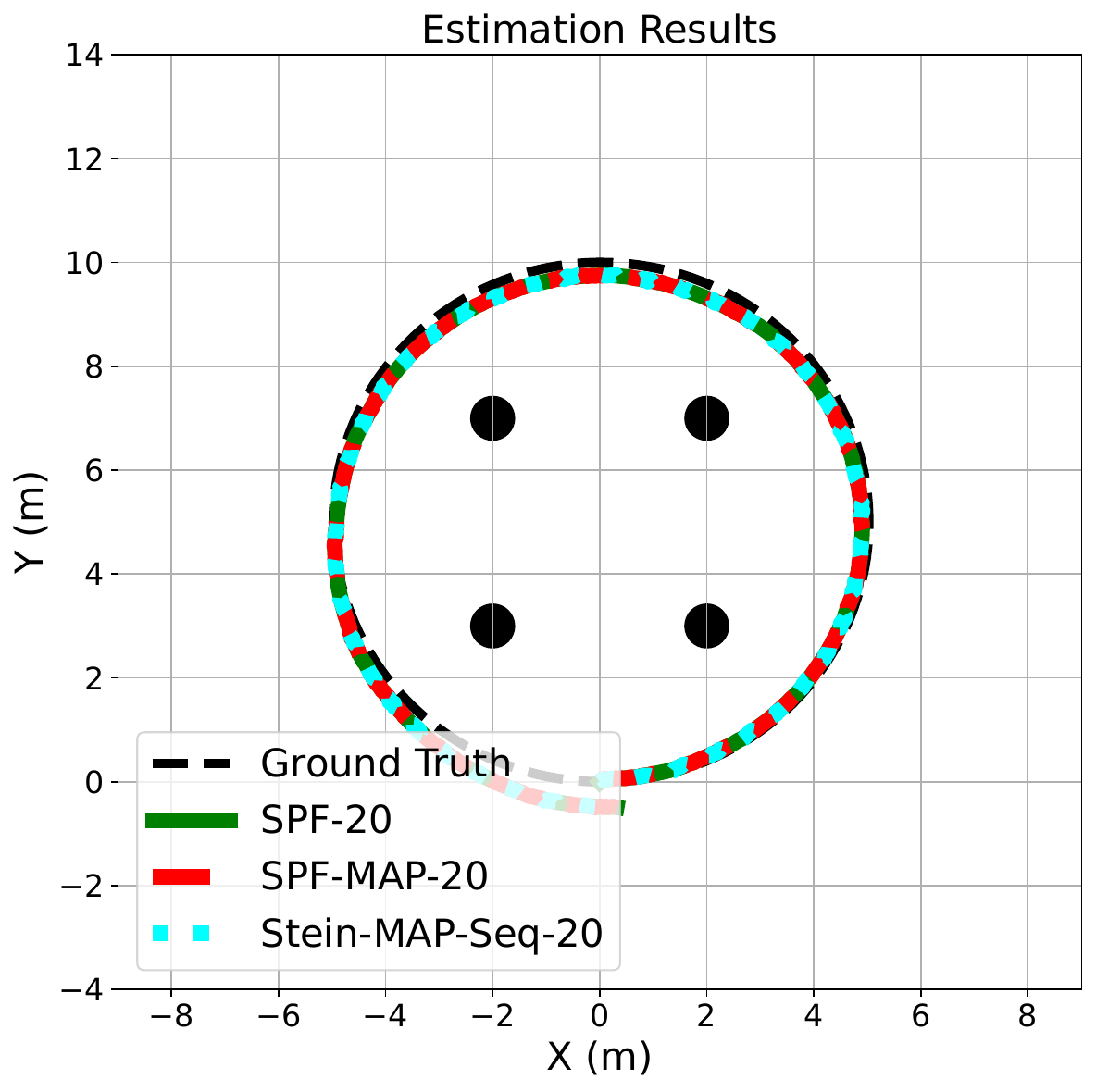}
        {{\scriptsize (d) Stein-based Estimation}}
    \end{minipage}
    \caption{{\small (a) A robot follows a counter-clockwise reference trajectory (black dashed curve) while executing noisy motion dynamics at every time step, causing the realized trajectory to gradually deviate from the noise-free reference. At each step, the robot receives a range–bearing measurement from one of four landmarks with unknown data association and heavy-tailed outliers, resulting in a highly multimodal measurement likelihood that varies across locations. (b) Trajectory estimates produced by Gaussian-assumed methods, illustrating bias and inconsistency under nonlinear motion and multimodal measurement likelihoods. (c) PF-based estimators with 500 particles improve robustness compared to Gaussian methods but still exhibit noticeable trajectory distortion due to weight degeneracy and resampling effects. (d) Stein-based estimators with 20 particles accurately recover the reference trajectory while remaining robust to multimodality and outliers, demonstrating reliable estimation with a small number of particles.}}
    \label{fig:Scenario2}
\end{figure*}

\begin{figure}[!t]
    \centering
    \includegraphics[width=0.99\linewidth]{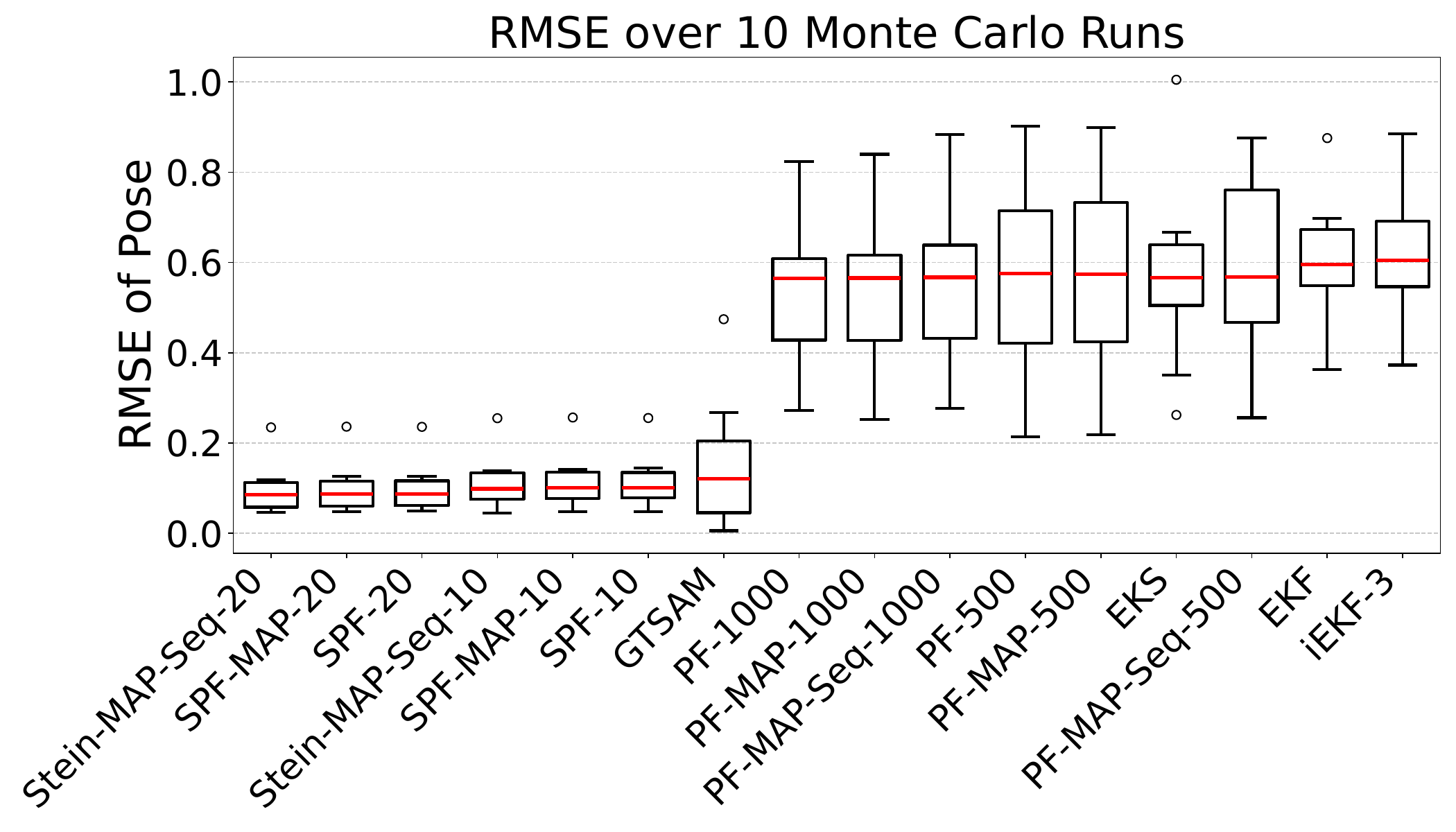}
    \caption{Results from $10$ Monte Carlo simulations over $630$ time steps are reported for pose estimation under unknown data association with heavy-tailed outliers. Stein-based methods demonstrate improved robustness and lower estimation error, whereas PF-based and Gaussian methods show increased sensitivity to ambiguous data association and heavy-tailed noise. Among Gaussian-assumed estimators, GTSAM benefits from trajectory-level smoothing, leading to improved performance relative to filtering-based approaches despite its reliance on unimodal assumptions.}
    \label{fig:Scenario2-2}
\end{figure}

To evaluate robustness under multimodal uncertainty, we consider a long-horizon localization task with nonlinear motion dynamics and constant turning. Translational and rotational process noise are injected at every time step, resulting in gradual uncertainty accumulation over time. At each step, the observed landmark is randomly selected, inducing abrupt changes in the measurement model and significantly reducing the temporal consistency of observations. In addition, measurement noise follows a heavy-tailed mixture distribution, introducing large outliers ($\times4$) with a $10\%$ probability in both range and bearing measurements. Collectively, these factors yield a likelihood function that is highly multimodal and heavy-tailed over time, posing substantial challenges for accurate and robust state estimation.

To evaluate the state estimation methods, we generated $10$ independent trajectories of $630$ time steps each. For each estimator, we computed the RMSE of the estimated states with respect to the noise-free reference trajectory and averaged the results over all Monte Carlo runs. The results in Fig.~\ref{fig:Scenario2}, Fig.~\ref{fig:Scenario2-2}, and Table~\ref{Table_Ex2} demonstrated that \textsf{Stein-MAP-Seq} consistently outperformed Gaussian-assumed methods, PF-based methods, and batch MAP estimation. This performance highlighted its strong robustness to highly multimodal and heavy-tailed posterior distributions arising from unknown data association with outliers.

\begin{table}[!t]
\begin{center}
\caption{Robustness Evaluation for Scenario (B), Averaged Pose RMSE}
\label{Table_Ex2}
\scriptsize 
\resizebox{0.48\textwidth}{!}{
    \begin{tabular}{c c c c|c c c c} 
    \hline
    \multicolumn{4}{c|}{\textbf{Gaussian-assumed Estimation}} 
    & \multicolumn{4}{c}{\textbf{Finite-state Estimation}} \\
    \cline{1-4}\cline{5-8}
    EKF & EKS & EKS-GT & GTSAM & $N_s$ & PF & PF-MAP & PF-MAP-Seq \\
    \hline
    &&&   & 5e2 & $0.5609$ & $0.5638$ & $0.5825$ \\
    $0.5903$ & $0.5442$ & $\mathbf{0.0179}$ & $0.1520$ & 1e3 & $\mathbf{0.5337}$ & $0.5370$ & $0.5512$ \\
    &&&   & 2e3 & $0.6223$ & $0.6239$ & $0.6457$ \\
    \hline
    iEKF-3 & iEKS-3 & iEKS-3-GT & GTSAM-Stein & $N_s$ & SPF & SPF-MAP & \textsf{Stein-MAP-Seq} \\
    \hline
    && & & 1e1 & $0.1146$ & $0.1142$ & $0.1121$ \\
    $0.6259$ & $0.6392$ & $0.0769$ & $0.1022$ & 2e1 & $0.0996$ & $0.0993$ & $\mathbf{0.0971}$\\
    && & & 4e1 & $0.0996$ & $0.0989$ & $0.0983$\\
    \hline
    \end{tabular}}
\end{center}
\normalsize
\end{table}

As illustrated in Fig.~\ref{fig:Scenario2}, (b) and Fig.~\ref{fig:Scenario2-2}, Gaussian-based estimators exhibited limited robustness in the presence of multi-peak posterior distributions. In particular, these methods tended to collapse onto a single mode, resulting in biased trajectory estimates. To isolate the effect of initialization in iterative smoothers, we additionally evaluated EKS-GT and iEKS-3-GT, both initialized at the noise-free ground-truth trajectory, as shown in Table~\ref{Table_Ex2}. While EKS-GT achieved accurate estimates, iEKS-3-GT produced slightly degraded results as the number of iterations increased. This behavior arose because repeated re-linearization around incorrect linearization points caused the estimate to drift toward spurious local modes.

As shown in Fig.~\ref{fig:Scenario2}, (b) and Table~\ref{Table_Ex2}, GTSAM improved trajectory smoothness by jointly optimizing the entire trajectory. In this simulation, since GTSAM does not support association-marginalized likelihoods, we adopted a nearest-landmark association strategy based on the observed measurement $\mathbf{z}_t$, rather than employing robust loss functions. Under this setting, GTSAM exhibited noticeable bias because high process uncertainty weakened the constraints imposed by the motion model, causing the estimator to rely more heavily on measurements. Consequently, when data associations were incorrect or measurements contained large outliers, these effects dominated the nonlinear least-squares objective. As a result, the optimizer often converged to a suboptimal local minimum that explained the measurements. These results indicated that the robustness of batch MAP estimation was limited in long-horizon localization problems characterized by high process uncertainty and outlier-contaminated observations.

As shown in Fig.~\ref{fig:Scenario2}, (c), PF-based estimators improved robustness compared to Gaussian-assumed methods by explicitly representing multiple hypotheses. With a sufficiently large number of particles, PF-based estimators were able to recover the trajectory and mitigate severe bias. However, despite this improvement, PF-based methods still exhibited noticeable trajectory distortion due to heavy-tailed outliers. This effect progressively reduced particle diversity over long horizons and led to premature mode collapse. Consequently, even with large particle sets, as reported in Table~\ref{Table_Ex2}, PF-based estimators struggled to maintain a coherent global trajectory, highlighting fundamental limitations in robustness.

As shown in Fig.~\ref{fig:Scenario2}, (d), Stein-based estimators consistently recovered the reference trajectory with high accuracy. Unlike PF-based methods, Stein-based approaches maintained a compact yet expressive particle set, effectively preserving multiple modes. This property enabled Stein-based estimators to avoid particle degeneracy and the associated trajectory distortions. As shown in Table~\ref{Table_Ex2}, accurate trajectory reconstruction was achieved with a substantially smaller number of particles, demonstrating both robustness and computational efficiency. These results highlighted the advantage of \textsf{Stein-MAP-Seq} in maintaining global trajectory consistency under persistent multimodality, making it particularly well-suited for long-horizon state estimation problems characterized by ambiguous measurements and heavy-tailed outliers.

\begin{figure}[t]
    \centering
    \includegraphics[width=0.99\linewidth]{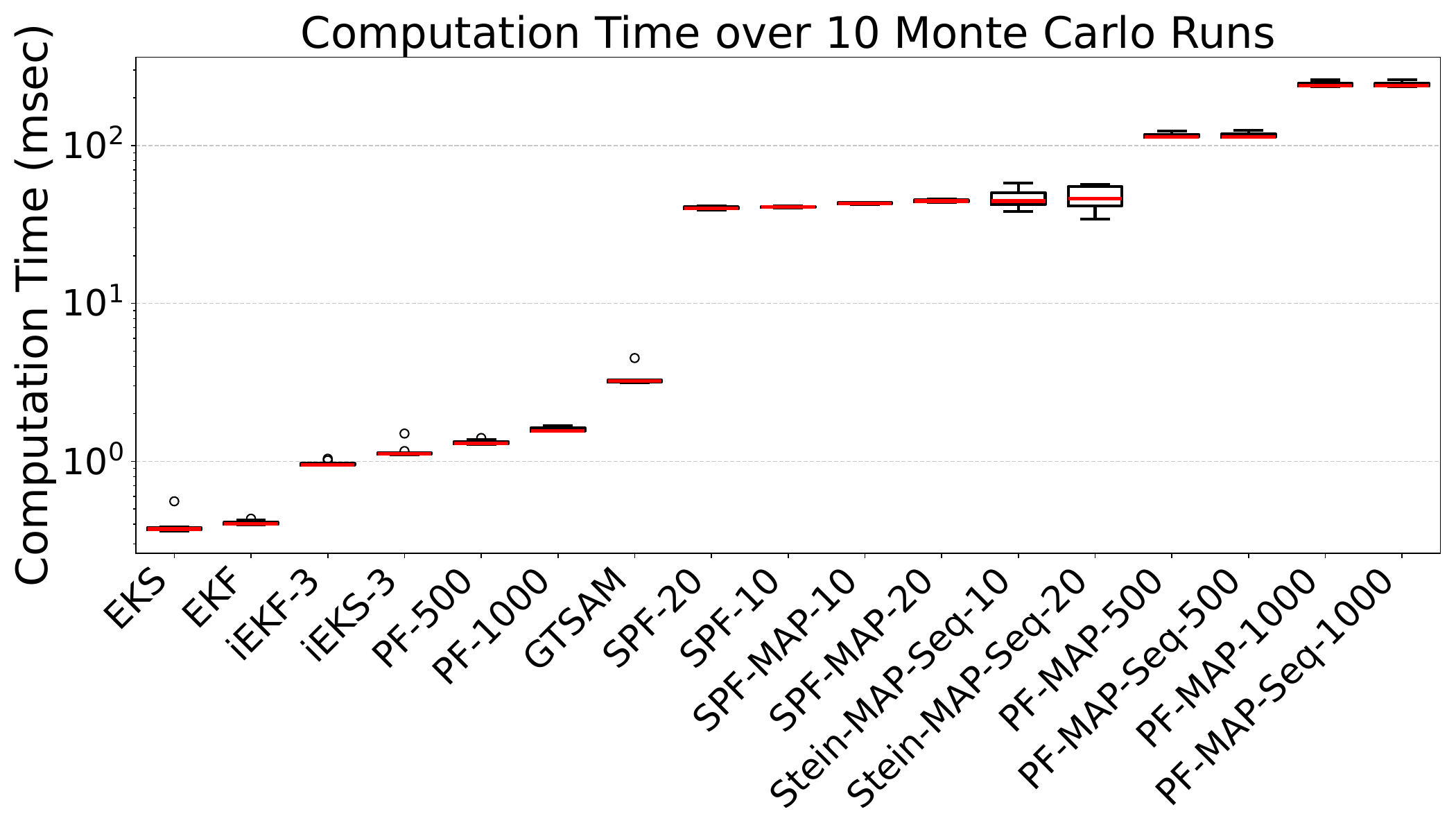}
    \caption{The averaged computation time per step is reported over $10$ Monte Carlo simulations, each consisting of $630$ time steps. Gaussian filtering and smoothing methods exhibit sub-millisecond runtimes due to their closed-form updates. GTSAM incurs a higher per-step computational cost than filtering–smoothing methods. PF-based MAP methods scale with the number of particles, resulting in rapidly increasing computation times due to forward recursion over large particle sets. Stein-based methods exhibit moderate computational costs by relying on a relatively small number of particles.}
    \label{fig:Scenario2-3}
\end{figure}

Alongside estimation accuracy, Fig.~\ref{fig:Scenario2-3} compared the per-step computational cost of the considered methods, revealing a clear accuracy–efficiency trade-off. Consistent with previous scenario, in the 3D case, Gaussian-assumed methods achieved the lowest computational cost due to closed-form updates and local temporal processing, but exhibited limited robustness under multimodal uncertainty. GTSAM incurred higher computational cost than filtering–smoothing methods, which improved robustness but remained sensitive to long-horizon multimodality. PF-based MAP estimators exhibited a rapid increase in computation time as the particle count increased and failed to maintain global trajectory coherence despite the higher cost. In contrast, \textsf{Stein-MAP-Seq} achieved strong robustness with moderate computational cost using significantly fewer particles, offering a favorable accuracy–efficiency trade-off in long-horizon, multimodal settings.

\begin{figure}[t]
    \centering
    \includegraphics[width=0.48\textwidth]{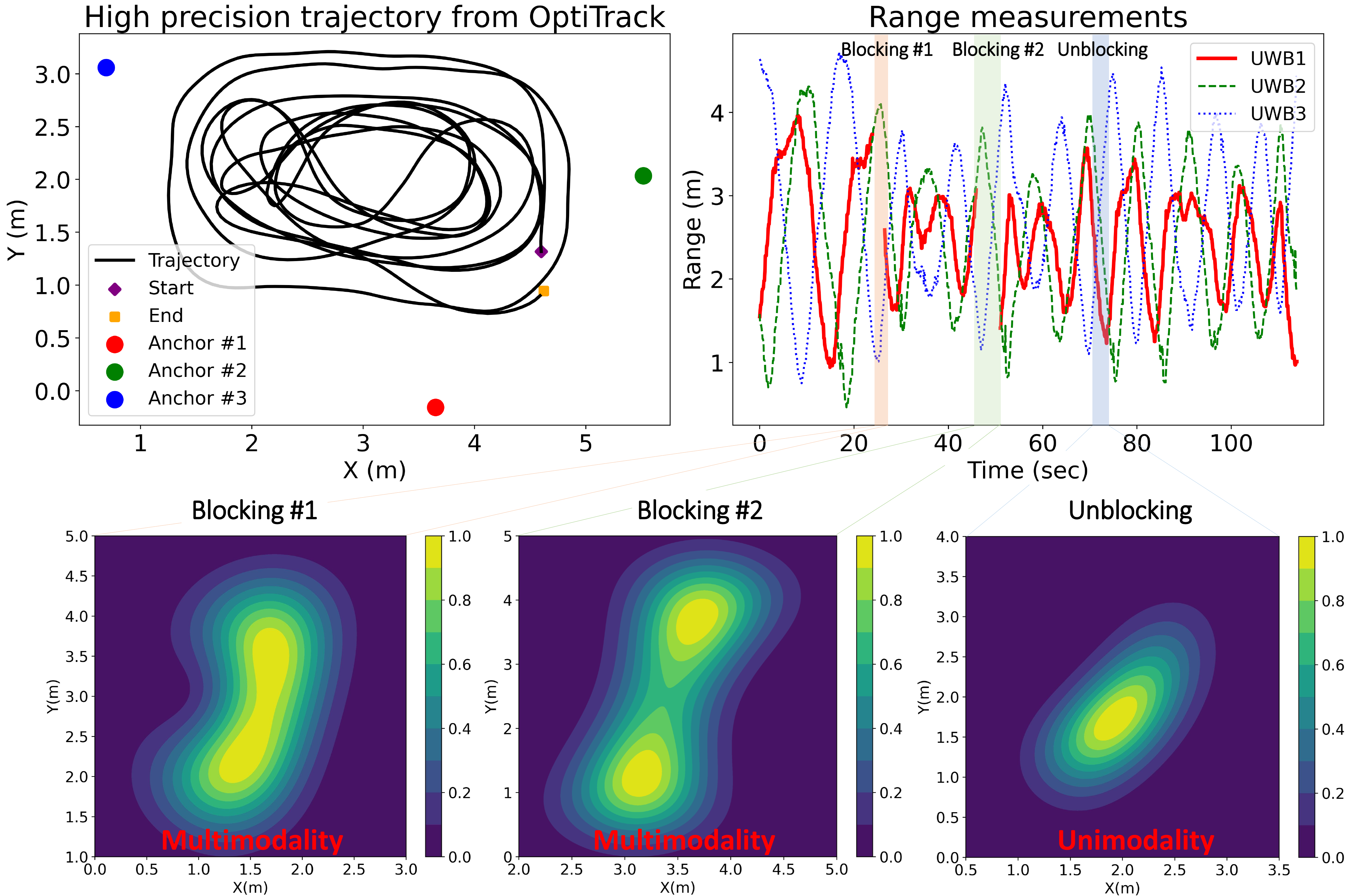}
    \caption{{\small A target traversed an arbitrary trajectory captured by the OptiTrack system at $120$Hz, while a UWB sensor provided range measurements at $10$Hz. During the $114.5\,s$ experiment, one tag and three anchors remained fully connected except during temporary blocking periods (i.e., intervals of unobservability). Anchor \#1 was blocked multiple times, leaving connections only to anchors \#2 and \#3, which induced multimodality in the measurement likelihood. The contour plots of the likelihood function during Blocking \#1 and Blocking \#2 show multimodal distributions, whereas the Unblocking period exhibits a unimodal distribution.}}
    \label{fig:Exp_scenario}
\end{figure}

\subsection{Range-only Localization under Temporary Unobservability}

In the third example, we consider a range-only localization task~\cite{olson2006robust,chen2022cooperative}. In our experiment, a target equipped with a DWM1000 Ultra-wideband (UWB) transceiver (tag) traversed along an \emph{arbitrary trajectory} in an $8\,$m $\times~8\,$m indoor environment with 3 UWB anchors. Here, UWB provides range measurements at 10 Hz and can also communicate their IDs, eliminating the need for data association considerations. There were no obstacles between the anchors and the tag to mitigate the biased range measurements. For a precise reference trajectory comparison, we employed the OptiTrack optical motion capture system with 12 infrared cameras to achieve high precision localization with an accuracy of $10^{-4}\;$m, and a sampling rate of $120$ Hz.

\begin{figure}[!t]
    \centering
    \includegraphics[width=0.42\textwidth]{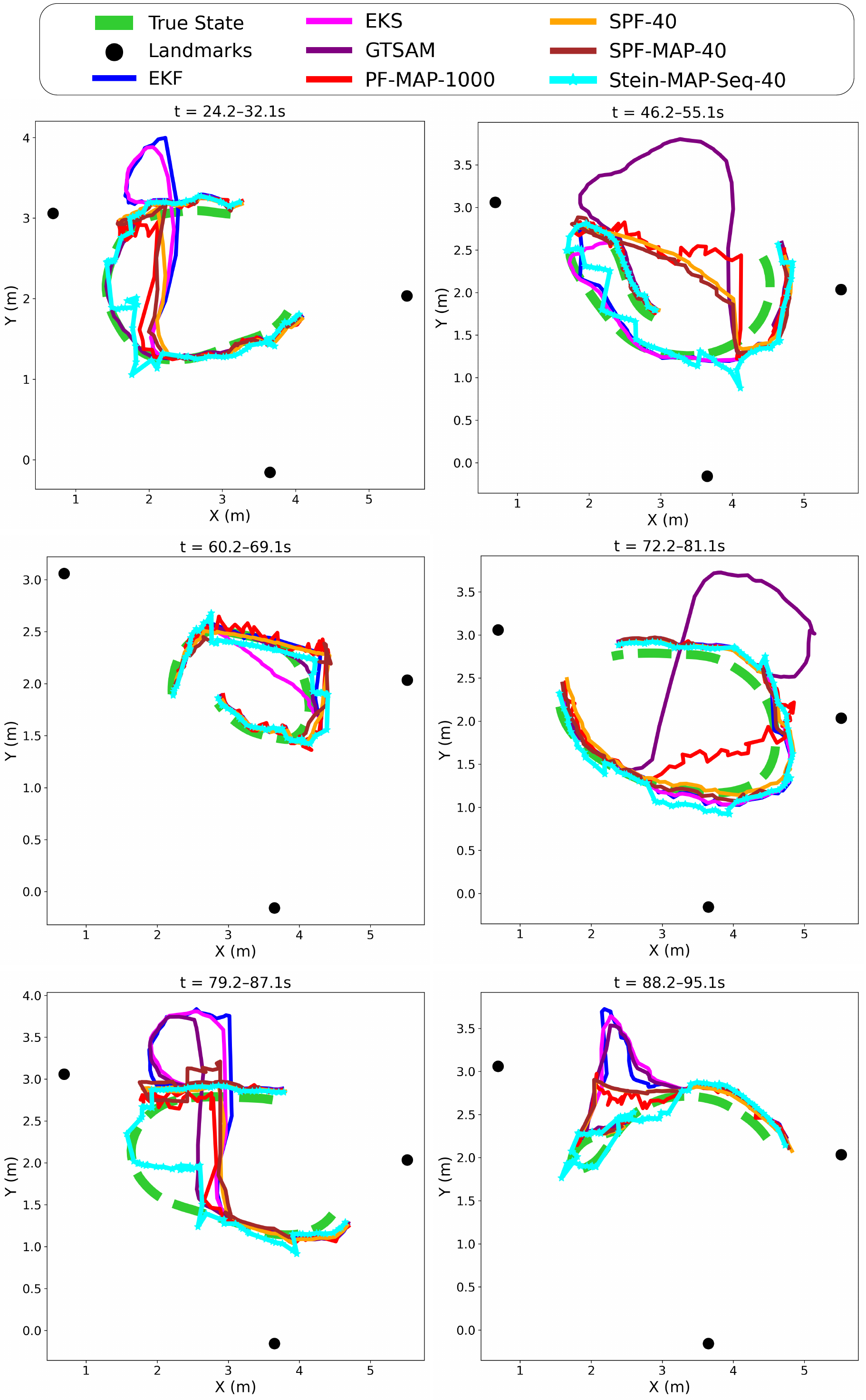}
    \caption{{\small Estimation results at $10$ Hz across $6$ blocking periods. EKF, PF, SPF, and GTSAM exhibit trajectory fragmentation at certain blocking intervals due to likelihood-driven updates under symmetric multimodal ambiguity. In contrast, \textsf{Stein-MAP-Seq} maintains stable MAP sequence estimates across all blocking periods.}}
    \label{fig:Scenario3}
\end{figure}

We localize a target that is \emph{not} equipped with any proprioceptive sensor, such as an IMU or wheel encoders, making the localization problem particularly challenging due to the absence of motion priors. Consequently, we adopt a zero-velocity motion model, given by  
\begin{align}
\mathbf{x}_t = \mathbf{x}_{t-1} + \mathbf{v}_{t-1},
\end{align}
where the process noise is modeled as $\mathbf{v}_{t-1} \sim \mathcal{N}(\mathbf{0}, \mathbf{Q}_{t-1})$, with $\mathbf{Q}_{t-1} = \mathrm{diag}\bigl[(\alpha_x\,\Delta t)^2,\;(\alpha_y\,\Delta t)^2\bigr]$. The state vector $\mathbf{x}_t = \bigl[x_t,\;y_t\bigr]^{\top} \in \real^2\,(m)$ represents the 2D position at time $t$. We use line-of-sight (LoS) range measurements of the form:
\begin{align}
    z_{t,l} = \bigl\lVert \Gamma_l - \mathbf{x}_t \bigr\rVert_{2} + r_{t,l},
\end{align}
where $z_{t,l} \in \mathbb{R}_{\geq 0}\,(m)$ is the LoS range to the $l$-th anchor at time $t$, $\Gamma_l = [\Gamma_{l,x}, \Gamma_{l,y}]^{\top} \in \real^2$ denotes the anchor position, $\|\cdot\|_2$ denotes the Euclidean norm, and the measurement noise is $r_{t,l}\sim \mathcal{N}(0, \sigma_z^2)$. The likelihood function over all range measurements is expressed as:
\begin{align}
    p(\mathbf{z}_t|\mathbf{x}_t) = \prod_{\forall l}\nolimits p(z_{t,l}|\mathbf{x}_t),
\end{align}
In this experiment, we set $\alpha_x=\alpha_y=1.0$, $\Delta t=0.1$, and $\sigma_{z} = 0.5$ to model measurement noise and compensate for minor range bias.

In this scenario, some anchors may be blocked by dynamic or static obstacles, leading to a multimodal distribution due to the underdetermined nature of distance measurements. In the experiment, we dropped measurements from anchor \#1 at specific time instances to simulate real-world scenarios for robustness analysis against a \emph{symmetric} multimodality, see Fig~\ref{fig:Exp_scenario}. 

We evaluated the accuracy of our positioning estimates against high-precision position data using the RMSE metric. Fig.~\ref{fig:Scenario3} shows the estimated positions at $10$ Hz produced by each method during each blocking period. The blocking durations were $2.5$s, $5.0$s, $4.0$s, $4.0$s, $3.0$s, and $3.5$s ($6$ periods in total). The high-precision trajectories, sampled at $120$ Hz, are shown as green dashed lines.

This scenario evaluated range-only localization under temporary unobservability, where measurements were intermittently unavailable, resulting in periods of weak observability and unknown motion. During these intervals, similar to previous scenarios, Gaussian-assumed methods quickly diverged due to linearization errors and their inability to represent multiple hypotheses.

In this scenario, GTSAM exhibited worse performance than filtering–smoothing methods. This behavior arose from the absence of effective motion constraints, combined with a lack of informative measurement factors during periods of temporary unobservability. As a result, GTSAM often converged to globally incorrect trajectories and failed to recover even after observability was restored.

PF-based methods exhibited different failure modes. PF methods relied on prior-based proposals with likelihood-weighted importance sampling, as shown in Fig.~\ref{fig:Scenario3-2}, which caused particles to prematurely collapse onto locally dominant but globally incorrect modes during unobservable intervals. Once particles were depleted, recovery remained difficult even after measurements became available again. Consequently, PF-MAP-Seq produced worse estimation results than both PF and PF-MAP due to frequent mode switching.

In contrast, \textsf{Stein-MAP-Seq} maintained multiple MAP-consistent hypotheses during unobservable periods. By conditioning on the previous MAP estimates, the proposed method remained robust to temporary loss of observability and consistently recovered the correct trajectory once measurements were reintroduced. Table~\ref{Table_Ex3} and Fig.~\ref{fig:Scenario3} demonstrated that \textsf{Stein-MAP-Seq} consistently maintained robustness against symmetric multimodal posteriors compared with the other methods.

Regarding computational cost, as shown in Table~\ref{Table_Ex3_2}, GTSAM incurred increased computation time due to a larger number of optimization iterations caused by the absence of informative motion priors and observations during the blocking period. In contrast, \textsf{Stein-MAP-Seq} achieved a moderate and predictable computational cost, consistent with its performance in the previous scenario, and outperformed other MAP-Seq estimators in terms of efficiency.

\begin{table}[t]
\begin{center}
\caption{Robustness Evaluation for Scenario (C), RMSE (m)}
\label{Table_Ex3}
\scriptsize 
\resizebox{0.48\textwidth}{!}{
    \begin{tabular}{c c c c|c c c c} 
    \hline
    \multicolumn{4}{c|}{\textbf{Gaussian-assumed Estimation}} 
    & \multicolumn{4}{c}{\textbf{Finite-state Estimation}} \\
    \cline{1-4}\cline{5-8}
    EKF & EKS & EKS-GT & GTSAM & $N_s$ & PF & PF-MAP & PF-MAP-Seq \\
    \hline
    &&&   & 1e3 & $0.2014$ & $0.1991$ & $0.4070$ \\
    $0.1951$ & $0.1955$ & $\mathbf{0.0886}$ & $0.2800$ & 2e3 & $0.2006$ & $\mathbf{0.1981}$ & $0.4217$ \\
    &&&   & 4e3 & $0.2019$ & $0.2003$ & $0.4244$ \\
    \hline
    iEKF-2 & iEKS-3 & iEKS-3-GT & GTSAM-Stein & $N_s$ & SPF & SPF-MAP & Stein-MAP-Seq \\
    \hline
    && & & 2e1 & $0.1921$ & $0.1969$ & $0.1722$ \\
    $0.1981$ & $0.1921$ & $0.1243$ & $0.1228$ & 3e1 & $0.1925$ & $0.1902$ & $0.1442$\\
    && & & 4e1 & $0.1949$ & $0.1858$ & $\mathbf{0.1369}$\\
    \hline
    \end{tabular}}
\end{center}
\normalsize
\end{table}

\begin{table}[t]
\begin{center}
\caption{Computation Time Evaluation for Scenario (C), (msec)}
\label{Table_Ex3_2}
\scriptsize 
\resizebox{0.48\textwidth}{!}{
    \begin{tabular}{c c c|c c c c} 
    \hline
    \multicolumn{3}{c|}{\textbf{Gaussian-assumed Estimation}} 
    & \multicolumn{4}{c}{\textbf{Finite-state Estimation}} \\
    \cline{1-3}\cline{4-7}
    EKF & EKS  & GTSAM & $N_s$ & PF & PF-MAP & PF-MAP-Seq \\
    \hline
    &&   & 1e3 & $2.463$ & $53.23$ & $58.94$ \\
    $0.1458$ & $0.1531$ & $36.69$  & 2e3 & $2.689$ & $148.12$ & $162.01$ \\
    &&   & 4e3 & $3.234$ & $460.36$ & $505.32$ \\
    \hline
    iEKF-2 & iEKS-3 & GTSAM-Stein & $N_s$ & SPF & SPF-MAP & Stein-MAP-Seq \\
    \hline
    &&  & 2e1 & $9.116$ & $9.736$ & $10.82$ \\
    $0.1905$ & $0.4505$ & $13.27$  & 3e1 & $11.60$ & $12.56$ & $13.60$\\
    &&  & 4e1 & $14.04$ & $15.28$ & $16.44$\\
    \hline
    \end{tabular}}
\end{center}
\normalsize
\end{table}

\begin{figure}[t]
    \centering
    \begin{minipage}[t]{0.475\textwidth}
        \centering
        \includegraphics[width=\textwidth]{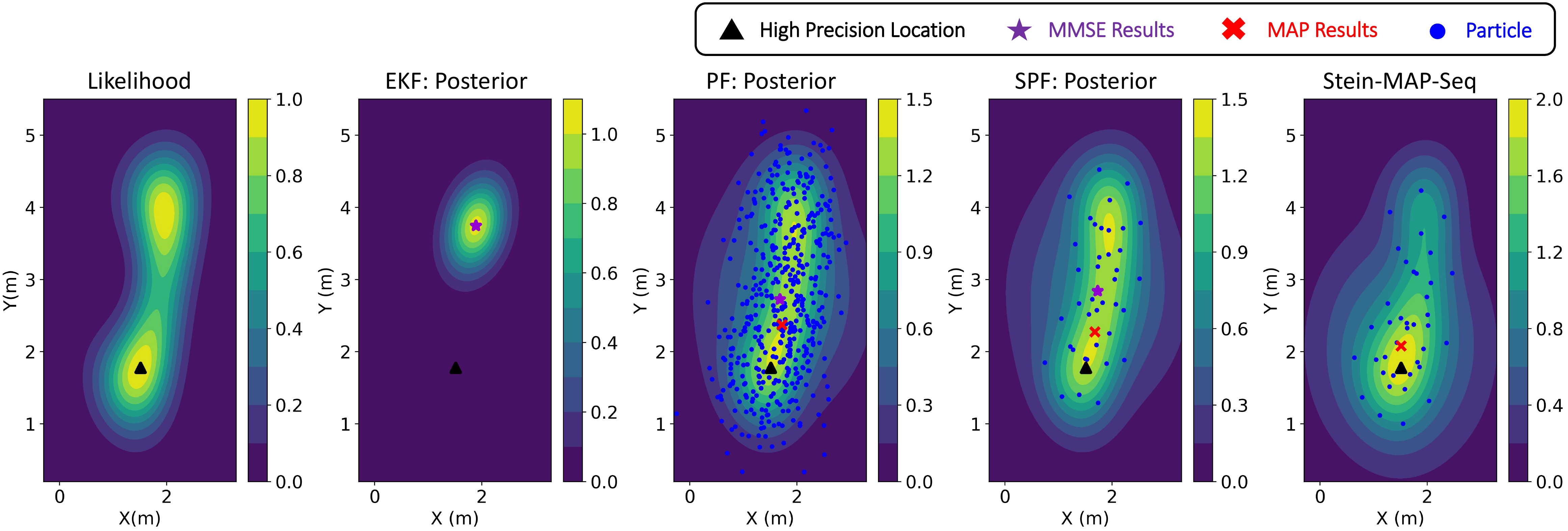}
        {{\small (a) Contour Plots @ $t=29.0s$ during Blocking \#1 period}}
    \end{minipage}
    \vskip\baselineskip
    \begin{minipage}[t]{0.475\textwidth}
        \centering
        \includegraphics[width=\textwidth]{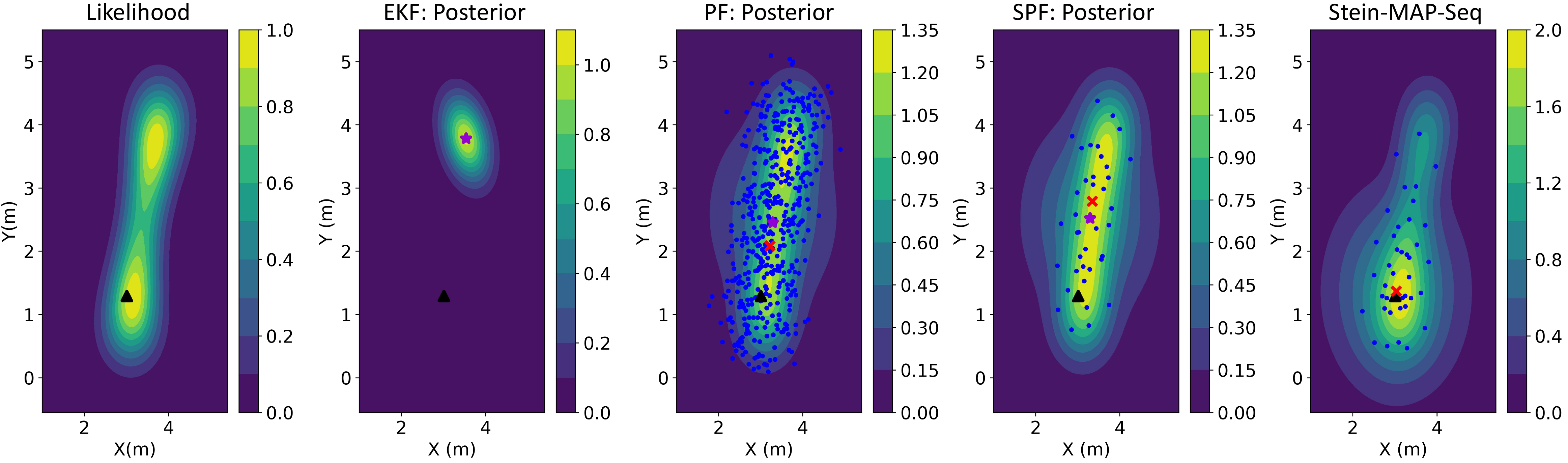}
        {{\small (b) Contour Plots @ $t=52.9s$ during Blocking \#2 period}}
    \end{minipage}
    \caption{{\small Posterior behaviors of EKF, PF, SPF and \textsf{Stein-MAP-Seq} during Blocking \#1 and \#2 periods. EKF collapses the posterior to a single incorrect mode due to local linearization, while PF relies on likelihood-based importance sampling, leading to premature mode commitment. SPF improves multimodal posterior representation but cannot handle symmetric uncertainty. In \textsf{Stein-MAP-Seq}, the displayed particles do \emph{not} represent posterior samples; instead, they visualize score-informed particle locations arising from the forward recursion toward a posterior defined conditionally on the previous MAP estimates $\mathbf{x}_{1:t-1}^{\star}$, thereby enabling accurate MAP sequence estimation.}}
    \label{fig:Scenario3-2}
\end{figure}

\subsection{High-Dimensional Manipulator with Kinematic Redundancy}

In the last example, we consider a \emph{high-dimensional nonlinear system} using a $7$-DoF robotic manipulator. The system state consists of joint positions and velocities, yielding a $14$-dimensional state vector, given by
\begin{align}
    \mathbf{x}_t =
    \begin{bmatrix}
        \mathbf{q}_t^\top &
        \dot{\mathbf{q}}_t^\top
    \end{bmatrix}^\top \in \mathbb{R}^{14},
\end{align}
where $\mathbf{q}_t, \dot{\mathbf{q}}_t \in \mathbb{R}^7$ denote the joint angles and joint angle velocities, respectively.

The manipulator dynamics are modeled using a discrete-time second-order system with sampling period $\Delta t = 0.02$,
\begin{subequations}
\begin{align}
    \mathbf{q}_t &= \mathbf{q}_{t-1} + \dot{\mathbf{q}}_{t-1} \Delta t, \\
    \dot{\mathbf{q}}_t &= \dot{\mathbf{q}}_{t-1} + \mathbf{u}_{t-1} \Delta t + \mathbf{v}_{t-1},
\end{align}
\end{subequations}
where $\mathbf{u}_{t-1} \in \mathbb{R}^7$ is the joint acceleration input and $\mathbf{v}_{t-1} \sim \mathcal{N}(\mathbf{0},\mathbf{Q}_{t-1})$ denotes Gaussian process noise. The covariance is chosen block-diagonal as $\mathbf{Q}_{t-1}=\mathrm{diag}(q_p\Delta t^2\mathbf{I}_7,~q_v\mathbf{I}_7)$ with $q_p=1.0$ and $q_v=5\times10^{-3}$. The control input is a smooth joint-wise sinusoid, $u_t^{(i)} = 0.5 \sin(0.2t + \phi_i),~~i=1,\dots,7$ where $\phi_i=i-1$ introduces fixed phase offsets across joints, resulting in rich excitation and strong nonlinear coupling in the forward kinematics.

At each time step, the system receives an exteroceptive observation of the end-effector position,
\begin{align}
    \mathbf{z}_t = \mathbf{h}(\mathbf{q}_t) + \mathbf{r}_t,
\end{align}
where $\mathbf{h}(\cdot)$ denotes the forward kinematics mapping and $\mathbf{r}_t \sim \mathcal{N}(\mathbf{0},\mathbf{R})$ with $\mathbf{R}=(0.1)^2\mathbf{I}_3$. The forward kinematics are defined by a Denavit–Hartenberg chain, yielding a nonlinear and kinematically redundant mapping from joint space to task space, which naturally induces multimodal posteriors in joint space as shown in Fig.~\ref{fig:Exp_scenario4}.

The initial state is distributed as $\mathbf{x}_0 \sim \mathcal{N}(\boldsymbol{\mu}_0,\boldsymbol{\Sigma}_0)$, where $\boldsymbol{\mu}_0$ initializes joint angles uniformly in $[-0.5,0.5]$ (rad) with zero velocities.
The initial covariance is block-diagonal, $\boldsymbol{\Sigma}_0=\mathrm{diag}(\sigma_q^2\mathbf{I}_7,~\sigma_{\dot{q}}^2\mathbf{I}_7)$, with $\sigma_q=1^\circ$ and $\sigma_{\dot{q}}=2^\circ$.

\begin{figure}[t]
    \centering
    \includegraphics[width=0.44\textwidth]{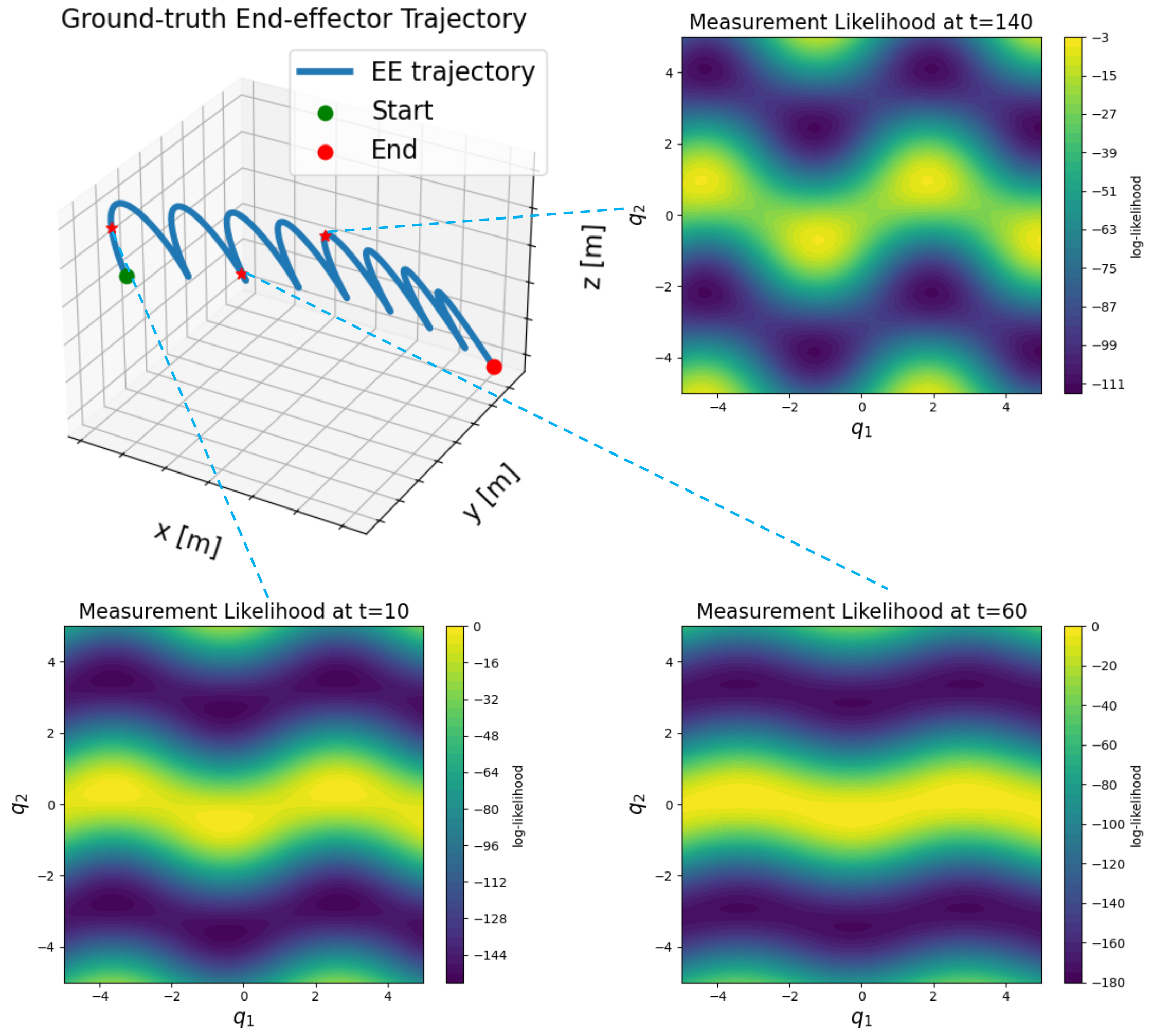}
    \caption{{\small Ground-truth end-effector trajectory in task space $\mathbb{R}^3$ (top left), together with measurement log-likelihood surfaces over the $(q_1,q_2)$ subspace at selected time steps (bottom and top right). Despite smooth motion in task space, the likelihood landscapes in joint space exhibit broad ridge-like structures and multiple local maxima, reflecting strong kinematic redundancy and multimodality.}}
    \label{fig:Exp_scenario4}
\end{figure}

\begin{figure}[!t]
    \centering
    \includegraphics[width=0.99\linewidth]{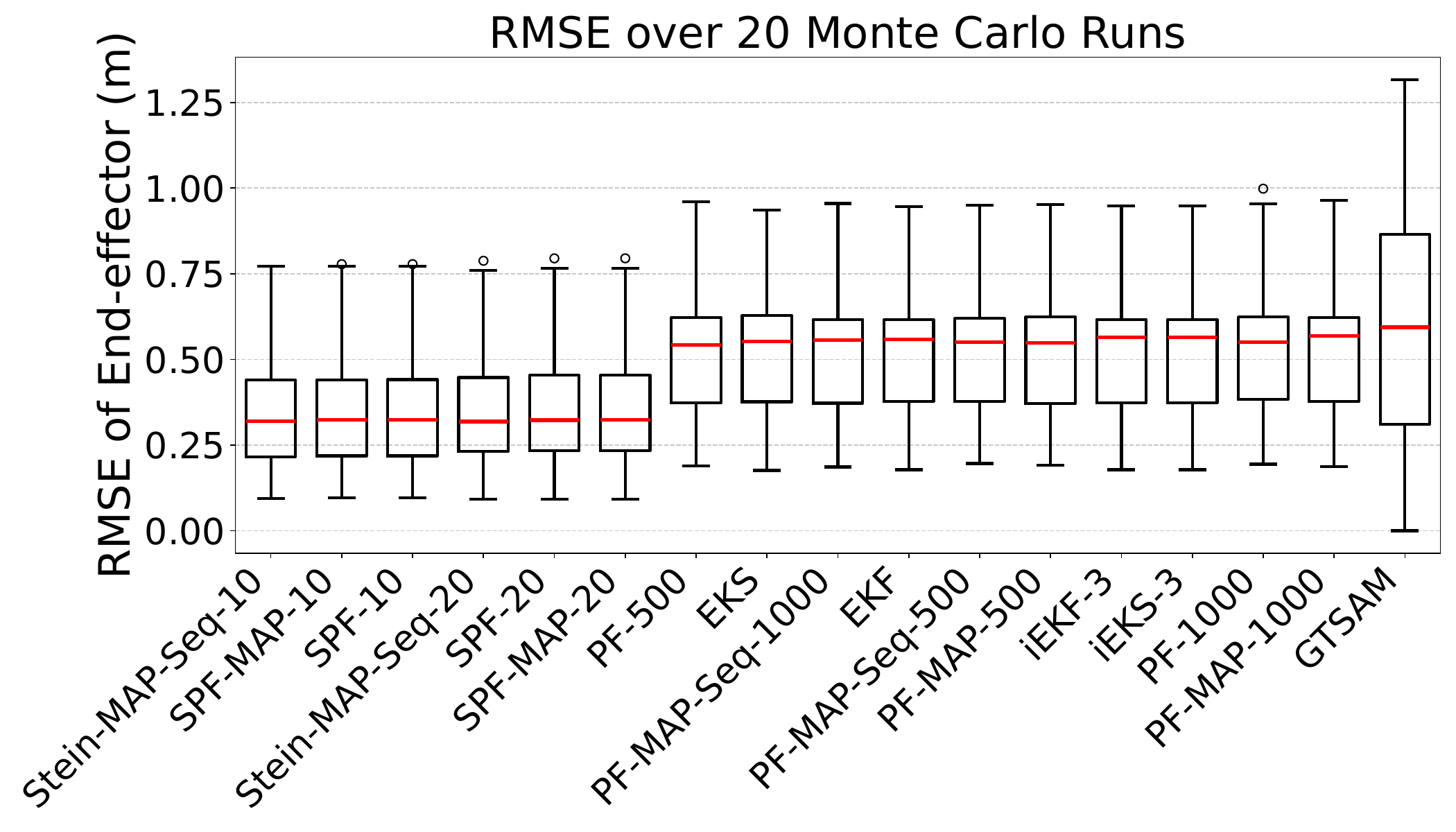}
    \caption{End-effector position RMSE over 20 Monte Carlo runs for all compared methods. All methods exhibit noticeable run-to-run variability due to kinematic redundancy and long-horizon error accumulation. Despite this variability, the Stein-based estimators achieve lower RMSE than the other methods.}
    \label{fig:Scenario4-2}
\end{figure}

Performance is evaluated using joint-space and end-effector RMSE, averaged over $20$ Monte Carlo runs. In this scenario, the $7$-DoF manipulator induces a highly multimodal posterior in joint space due to kinematic redundancy and end-effector–only observations. As illustrated in Fig.~\ref{fig:Exp_scenario4}, smooth and well-behaved motion in task space gives rise to broad, ridge-like likelihood structures with multiple local maxima in joint space, making state estimation particularly challenging.

Quantitative results in Fig.~\ref{fig:Scenario4-2} and Table~\ref{Table_Ex4} showed that \textsf{Stein-MAP-Seq} consistently achieved the lowest end-effector RMSE among all compared methods. All methods exhibited large variance because multiple joint configurations were consistent with the same end-effector measurement, resulting in strong multimodality. Unlike previous scenarios, GTSAM showed inferior performance under kinematic redundancy.

PF-based estimators required a substantially larger number of particles to maintain trajectory coherence in this $14$-dimensional state space. As shown in Fig.~\ref{fig:Scenario4-2}, even with an increased particle count, PF-based MAP estimators suffered from degraded accuracy, reflecting particle path degeneracy and unstable mode selection. In contrast, \textsf{Stein-MAP-Seq} achieved superior performance using as few as $10$–$30$ particles, highlighting the advantage of mode-seeking behavior combined with forward recursion in maintaining mode-coherent support for MAP sequence estimation, even in high-dimensional scenarios.

Joint-space RMSE results in Table~\ref{Table_Ex4_2} further revealed an important distinction between task-space and configuration-space accuracy. Even when comparable end-effector RMSE was achieved, most estimators exhibited significantly larger joint-space errors, indicating inconsistency in resolving kinematic redundancy. In contrast, \textsf{Stein-MAP-Seq} achieved the lowest joint-space RMSE among MAP-Seq estimators, demonstrating its ability to maintain consistency in both task space and configuration space. Moreover, the trajectory estimates produced by \textsf{Stein-MAP-Seq} provided effective initialization for GTSAM, guiding the optimizer toward correct joint-space modes and thereby improving configuration-space estimation accuracy.

Finally, the computational comparison in Fig.~\ref{fig:Scenario4-3} showed that \textsf{Stein-MAP-Seq} achieved a favorable accuracy–efficiency trade-off. Despite operating in a $14$-dimensional state space, a small number of particles was sufficient to represent dominant modes, thereby maintaining coherent MAP trajectory hypotheses. Moreover, even as the state dimension $n_x$ increased, particle updates were performed in parallel on computing hardware, maintaining computational cost comparable to that of lower-dimensional scenarios. Consequently, the runtime remained comparable to GTSAM and was substantially lower than that of PF-based MAP estimators. These results demonstrated that \textsf{Stein-MAP-Seq} scaled effectively to high-dimensional, kinematically redundant systems and provided robust MAP-Seq estimation.

\begin{table}[t]
\begin{center}
\caption{\scriptsize Robustness Evaluation for Scenario (D), End-effector RMSE (m)}
\label{Table_Ex4}
\scriptsize 
\resizebox{0.48\textwidth}{!}{
    \begin{tabular}{c c c c|c c c c} 
    \hline
    \multicolumn{4}{c|}{\textbf{Gaussian-assumed Estimation}} 
    & \multicolumn{4}{c}{\textbf{Finite-state Estimation}} \\
    \cline{1-4}\cline{5-8}
    EKF & EKS & EKS-GT & GTSAM & $N_s$ & PF & PF-MAP & PF-MAP-Seq \\
    \hline
    &&&   & 5e2 & $0.6115$ & $0.6134$ & $0.6120$ \\
    $0.6078$ & $0.6007$ & $\mathbf{0.0617}$ & $0.9198$ & 1e3 & $0.5946$ & $0.5912$ & $0.5789$ \\
    &&&   & 2e3 & $0.5409$ & $0.5393$ & $\mathbf{0.5370}$ \\
    \hline
    iEKF-3 & iEKS-3 & iEKS-3-GT & GTSAM-Stein & $N_s$ & SPF & SPF-MAP & Stein-MAP-Seq \\
    \hline
    && & & 1e1 & $0.4712$ & $0.4708$ & $0.4686$ \\
    $0.6087$ & $0.6087$ & $0.1242$ & $0.5647$ & 2e1 & $0.4652$ & $0.4652$ & $0.4624$\\
    && & & 3e1 & $0.3584$ & $0.3584$ & $\mathbf{0.3560}$\\
    \hline
    \end{tabular}}
\end{center}
\normalsize
\end{table}

\begin{table}[t]
\begin{center}
\caption{\scriptsize Robustness Evaluation for Scenario (D), Joint Angles $\mathbf{q}_t$ RMSE (rad)}
\label{Table_Ex4_2}
\scriptsize 
\resizebox{0.48\textwidth}{!}{
    \begin{tabular}{c c c c|c c c c} 
    \hline
    \multicolumn{4}{c|}{\textbf{Gaussian-assumed Estimation}} 
    & \multicolumn{4}{c}{\textbf{Finite-state Estimation}} \\
    \cline{1-4}\cline{5-8}
    EKF & EKS & EKS-GT & GTSAM & $N_s$ & PF & PF-MAP & PF-MAP-Seq \\
    \hline
    &&&   & 5e2 & $0.8258$ & $0.9622$ & $0.9581$ \\
    $0.8534$ & $0.9132$ & $\mathbf{0.1334}$ & $4.4841$ & 1e3 & $0.7792$ & $0.7942$ & $0.8465$ \\
    &&&   & 2e3 & $\mathbf{0.7329}$ & $0.7713$ & $0.8668$ \\
    \hline
    iEKF-3 & iEKS-3 & iEKS-3-GT & GTSAM-Stein & $N_s$ & SPF & SPF-MAP & Stein-MAP-Seq \\
    \hline
    && & & 1e1 & $0.1922$ & $0.1921$ & $0.1918$ \\
    $0.8668$ & $1.0430$ & $ 0.2320$ & $0.9636$ & 2e1 & $0.1848$ & $0.1848$ & $0.1843$\\
    && & & 3e1 & $0.1728$ & $0.1728$ & $\mathbf{0.1720}$\\
    \hline
    \end{tabular}}
\end{center}
\normalsize
\end{table}

\medskip

\subsection{Discussion on Empirical Performance and Design Insights of \textsf{Stein-MAP-Seq}}
\label{subsec::E}

This subsection discusses the empirical performance of \textsf{Stein-MAP-Seq} with a focus on the role of kernel choice, particle count, and kernel hyperparameters. Beyond reporting estimation accuracy, our goal is to provide practical insights into why \textsf{Stein-MAP-Seq} performs favorably in multimodal MAP-Seq estimation, clarify its limitations, and discuss practical considerations arising from modeling assumptions in real robotic systems.

\subsubsection{Effect of Kernel Choice}
The kernel function plays a fundamentally different role in MAP-Seq estimation than in MMSE-based estimation. While MMSE-oriented methods aim to approximate the full posterior distribution and are therefore sensitive to kernel-induced distributional shaping, MAP-Seq estimation primarily benefits from coherent particle concentration around dominant modes rather than uniform posterior coverage. To assess the impact of kernel choice, we conducted an ablation study comparing RBF, IMQ, and Matérn ($\nu=2/3$) kernels~\cite{gorham2017measuring,seeger2004gaussian}, all of which preserve the same parallelizable pairwise interaction structure while introducing heavy-tailed repulsion (IMQ) and more localized particle interactions with finite correlation length (Matérn), respectively. As reported in Table~\ref{Table_Ex6}, performance differences across kernels are relatively small, indicating that kernel shape alone is not a dominant factor in MAP-Seq estimation.

\begin{figure}[!t]
    \centering
    \includegraphics[width=0.99\linewidth]{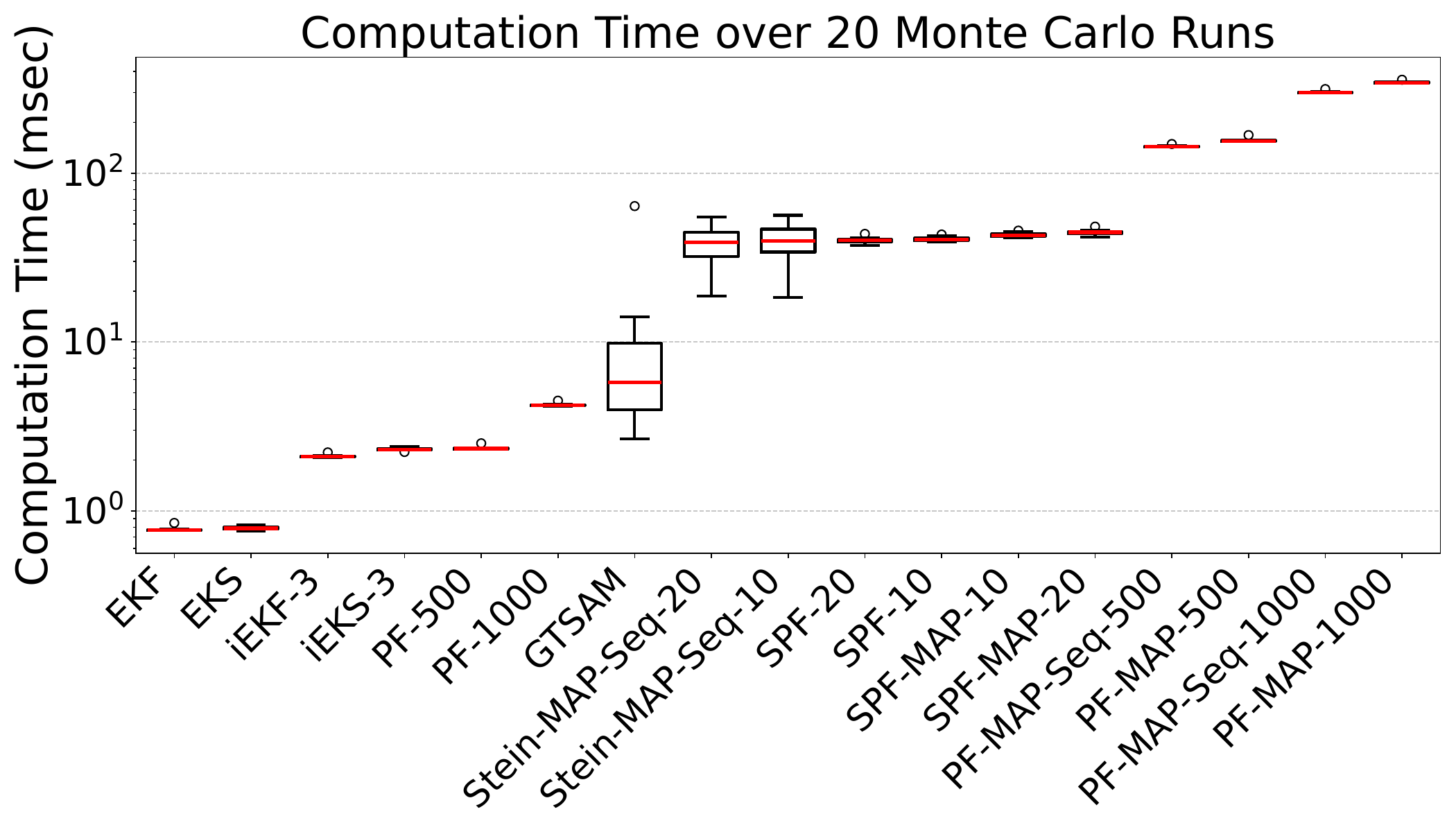}
    \caption{The averaged computation time per time step is reported over $20$ Monte Carlo simulations, each consisting of $250$ time steps. GTSAM exhibits variable computation time due to differences in convergence behavior across runs. \textsf{Stein-MAP-Seq}, despite operating in a high-dimensional state space, achieves comparable computation time by using a small number of particles and fully parallelizable particle updates. In contrast, PF-based MAP sequence estimators incur the highest computational cost as a result of their reliance on large particle counts.}
    \label{fig:Scenario4-3}
\end{figure}

\subsubsection{Estimation Accuracy and Particle Count}
Unlike MC sampling-based MAP-Seq estimators, \textsf{Stein-MAP-Seq} does not exhibit monotonic performance improvement with increasing particle count. Empirically, a relatively small number of particles (typically $10-50$), when combined with an appropriate kernel bandwidth, suffices to achieve accurate MAP trajectory recovery across all tested scenarios. As the particle count increases, the repulsive interaction induced by the kernel term in Eq.~\eqref{eq::SVGD_newupdate} becomes more pronounced, potentially spreading particles into lower-density regions of the state space. While such diversity is advantageous for MMSE-based estimation, it can adversely affect MAP-Seq estimation by increasing path ambiguity during the forward recursion in Eq.~\eqref{eq::forward_recursion}. Consequently, excessive particle counts may reduce trajectory coherence and lead to increased estimation error. This trend is quantitatively supported by Table~\ref{Table_Ex6}, which shows that increasing the particle count yields only marginal performance improvements or even degrades estimation accuracy. These results highlight a fundamental distinction between \textsf{Stein-MAP-Seq} and MC sampling-based MAP-Seq estimators, which typically rely on large particle sets.

\subsubsection{Kernel Bandwidth and Particle Count}
Kernel bandwidth controls the trade-off between mode concentration and separation by adjusting particle repulsion. Smaller bandwidths yield more concentrated particle clusters, whereas larger bandwidths induce stronger repulsive interactions and improved mode separation. MAP-Seq estimation is therefore more sensitive to bandwidth choice than MMSE-based estimation, reflecting its mode-seeking nature, where trajectory decoding depends on temporally consistent relative scores rather than marginal posterior coverage. These trends are quantitatively confirmed in Table~\ref{Table_Ex6}. For the RBF kernel, increasing the bandwidth from $0.5h$ to $3h$ substantially reduces RMSE even with fewer particles, while smaller bandwidths require larger particle sets to achieve comparable performance. Overall, kernel bandwidth and particle count jointly shape trajectory-level score accumulation, providing practical design guidance for MAP-Seq estimation.

\begin{table}[t]
\centering
\caption{Ablation study on kernel type, kernel bandwidth, and particle count.
Results are averaged over Monte Carlo runs in Scenario~A.
Unless otherwise specified, median heuristic bandwidth $h$ and $N_s=10$ are used.}
\label{Table_Ex6}
\scriptsize
\resizebox{0.49\textwidth}{!}{
\begin{tabular}{c c c c}
\hline
Ablation Factor & Kernel & Setting & RMSE \\
\hline
Kernel type     & RBF                & $N_s=10$      & 2.2643 \\
Kernel type     & IMQ                & $N_s=10$      & 2.2512 \\
Kernel type     & Matérn             & $N_s=10$      & 2.2344 \\
\hline
Bandwidth scale & RBF                & $0.5h, N_s=10$        & 2.4887 \\
Bandwidth scale & RBF                & $3h, N_s=10$          & 2.0462 \\
Bandwidth scale \& Particle count & RBF                & $0.5h, N_s=40$        & 2.2507 \\
Bandwidth scale \& Particle count & RBF                & $3h, N_s=40$          & 2.1970 \\
\hline
Particle count  & RBF                & $N_s=20$      & 2.2108 \\
Particle count  & RBF                & $N_s=40$      & 2.1369 \\
Particle count  & IMQ                & $N_s=20$      & 2.1557 \\
Particle count  & IMQ                & $N_s=40$      & 2.1780 \\
Particle count  & Matérn             & $N_s=20$      & 2.1232 \\
Particle count  & Matérn             & $N_s=40$      & 2.0714 \\
\hline
\end{tabular}
}
\end{table}

\subsubsection{Practical Considerations Beyond Ideal Smoothness and Independence Assumptions} 

\textsf{Stein-MAP-Seq} assumes locally smooth system dynamics and conditionally independent observations; however, these assumptions are often violated in real robotic systems due to nonsmooth or discontinuous dynamics and temporally correlated high-rate sensor measurements. Importantly, \textsf{Stein-MAP-Seq} does not rely on global smoothness or strict observation independence. Instead, it operates through local gradient-based particle updates and trajectory-level score accumulation, making it tolerant to moderate deviations from these assumptions, with locally defined or surrogate gradients being sufficient for MAP-Seq estimation. In practice, local nonsmoothness typically affects only short time intervals and does not invalidate the overall MAP trajectory as long as the dominant mode remains temporally consistent, while mild observation correlations mainly change how confident the estimate is, rather than where the mode is. Such effects often manifest as noisier gradients or increased uncertainty and can be mitigated through kernel bandwidth adjustment and particle interaction regularization. As with classical filtering and batch MAP methods, extreme violations may degrade performance, and characterizing behavior under severe nonsmooth dynamics and strongly correlated observations remains an important direction for future work.

\subsubsection{Complementarity with Batch MAP Optimization}
The comparison with batch MAP optimization using GTSAM reveals complementary strengths. While GTSAM jointly optimizes the entire trajectory, its performance is sensitive to initialization under strongly multimodal and nonconvex objectives. In our demonstrations, initializing GTSAM with the \textsf{Stein-MAP-Seq} trajectory guided the optimizer toward the correct mode, resulting in faster convergence and improved estimation accuracy. These results suggest that \textsf{Stein-MAP-Seq} can serve not only as a standalone MAP-Seq estimator, but also as an effective global mode-discovery front-end for multimodal batch MAP optimization.


\section{Conclusion} \label{sec::conclusion}
State estimation in autonomous systems remains a fundamental challenge, particularly in real-world settings characterized by complex dynamics and highly multimodal uncertainty. In this work, we introduced a computationally efficient MAP sequence estimation framework that explicitly addresses multimodal posterior distributions while reducing computational and storage requirements. The proposed approach is grounded in a sequential variational inference formulation that explicitly accounts for temporal dependencies induced by system dynamics.

Building on this formulation, we addressed MAP sequence estimation through a parallelizable particle-based gradient flow embedded within a Viterbi-style dynamic programming framework. This combination enables efficient mode concentration and globally consistent trajectory selection over long horizons. We rigorously evaluated the proposed method in a range of challenging multimodal scenarios, including nonlinear systems with ambiguous measurements, pose estimation under unknown data association, range-only localization under temporary unobservability, and high-dimensional manipulator estimation with kinematic redundancy.

Across all scenarios, \textsf{Stein-MAP-Seq} consistently demonstrated strong robustness and scalability. Moreover, beyond empirical accuracy, this work provides practical design insights for MAP sequence estimation in multimodal settings. Our results indicate that MAP-Seq estimation benefits primarily from coherent mode concentration and temporal consistency, rather than accurate approximation of the full posterior distribution. Accordingly, \textsf{Stein-MAP-Seq} achieves reliable MAP trajectory recovery using a small number of finite states, with performance governed more by kernel-induced mode structure than particle count alone. Moreover, the proposed method complements batch MAP optimization by providing robust global initialization under strongly multimodal and nonconvex objectives, highlighting its utility both as a standalone estimator and as a front-end for existing optimization-based frameworks.


\appendix

\subsection{Auxiliary results and proofs}\label{sec::appen_A}
\begin{lem}\label{lem::ELBO} {(Evidence Lower Bound). The relationship between proposal distribution $q(x_{0:T})$ and $p(x_{0:T}| z_{1:T})$ can be derived as
\begin{align} \label{eq::KLD_4}
 \!\!\!\text{KL} \bigr[q(x_{0:T})\parallel p(x_{0:T}|z_{1:T}) \bigr] \!=\! -\mathcal{L}(x_{0:T}) \!+\! \log p(z_{1:T}),
\end{align}}
where $ \mathcal{L}(x_{0:T})$ us given in~\eqref{eq::ELBO_term}.
\end{lem}
\begin{proof}
 To address the connection between $p(x_{0:T}| z_{1:T})$ and $q(x_{0:T})$, the representation of $\log p(x_{0:T}| z_{1:T})$ can be established using the relationship between joint and conditional probabilities given by
\begin{align}
\label{eq::KLD_1}
    \log p(x_{0:T}| z_{1:T}) = \log p(x_{0:T}, z_{1:T}) - \log p(z_{1:T}),
\end{align}
The proposal distribution $q(x_{0:T})$ is then introduced by subtracting of $\log q(x_{0:T})$ from both sides of~\eqref{eq::KLD_1}, that is,
\begin{align}
\label{eq::KLD_2}
    \!\!\log \frac{p(x_{0:T}| z_{1:T})}{q(x_{0:T})} = \log \frac{p(x_{0:T}, z_{1:T})}{q(x_{0:T})} - \log p(z_{1:T}).
\end{align}
Note that the expectation of both sides of~\eqref{eq::KLD_2} w.r.t. $q(x_{0:T})$,
\begin{align}
\label{eq::KLD_3}
    &\int q(x_{0:T})\,\log\frac{p(x_{0:T}| z_{1:T})}{q(x_{0:T})} dx_{0:T}  \\
    &~~= \int q(x_{0:T})\, \log \frac{p(x_{0:T}, z_{1:T})}{q(x_{0:T})} dx_{0:T} - \log p(z_{1:T}), \nonumber
\end{align}
where $\log p(z_{1:T})$ on the right-hand side of~\eqref{eq::KLD_3} is independent of $x_{0:T}$. Therefore, $\int q(x_{0:T})\log p(z_{1:T})dx_{0:T}=\log p(z_{1:T})$. Moreover, the left-hand side of~\eqref{eq::KLD_3} can be expressed using the KL divergence:
\begin{align}
    &\text{KL} \bigr[q(x_{0:T}) \parallel p(x_{0:T}|z_{1:T})\bigr]  \\
    &= -\int q(x_{0:T})\, \log\frac{p(x_{0:T}, z_{1:T})}{q(x_{0:T})} dx_{0:T} + \log p(z_{1:T}) \geq 0. \nonumber
\end{align}
\end{proof}

\subsection{Proof of Lemma~\ref{lem::proposal}}\label{sec::appen_B}
The optimal distribution $q^{\star}(x_{0:T})$ is given as
\begin{align}
  q^{\star}(x_{0:T}) = \arg \! \max_{\!\!\!\!\!q(x_{0:T})}\mathcal{L}(x_{0:T}),
\end{align}
where
\begin{align} \label{eq::App_ELBO}
  \mathcal{L}(x_{0:T}) &= \int q(x_{0:T}) \log p(x_0) dx_{0:T} \nonumber \\
  &~~~~+ \int q(x_{0:T}) \sum\nolimits_{t=1}^T \log p(x_t|x_{t-1}) dx_{0:T} \nonumber \\
  &~~~~+ \int q(x_{0:T}) \sum\nolimits_{t=1}^T \log p(z_t|x_t) dx_{0:T} \nonumber \\ &~~~~- \int q(x_{0:T})\log q(x_{0:T}) dx_{0:T}.
\end{align}
The last integral term in~\eqref{eq::App_ELBO} can be represented by
\begin{align}
    - &\int q(x_{0:T})\log q(x_{0:T}) dx_{0:T} = -\int q(x_0)\log q(x_0) dx_0 \nonumber \\&- \sum\nolimits_{t=2}^T \int q(x_{0:T})\log q(x_t|x_{t-1},x_{0:t-2}) dx_{0:T}.
\end{align}
Using the conditional differential entropy~\cite{ash2012information} as given by
\begin{align}
    \!\!\!\!\!-\int \!\!\! \int p(x,y) \log p(x|y) dx dy \leq\! -\! \int\! p(x)\log p(x) dx,
\end{align}
where equality if and only if $x$ and $y$ are independent, we have
\begin{align} \label{eq::App_cde}
    - \int q(x_{0:T})\log q(x_t|x_{t-1},x_{0:t-2}) dx_{0:T} \nonumber \\ \leq - \int q(x_{0:T})\log q(x_t|x_{t-1}) dx_{0:T}.
\end{align}
Under Assumption~\ref{assump::MAP}, there exists
\begin{align}
    q(x_t|x_{t-1}) = q(x_t|x_{t-1},x_{0:t-2}),
\end{align}
and thus $q(x_{0:T})$ is given via~\eqref{eq::App_cde} as 
\begin{align} 
    q(x_{0:T}) = q(x_0)\prod\nolimits_{t=1}^{T} q(x_t|x_{t-1}).
\end{align}

\subsection{Proof of Lemma~\ref{lem::factKL}} \label{sec::appen_C}
Under Lemma~\ref{lem::proposal}, and considering that $\log p(x_t|x_{t-1})$ and $\log p(z_t|x_t)$ depend solely on $x_t$ and/or $x_{t-1}$ of $x_{0:T}$, the ELBO in~\eqref{eq::ELBO_term} is derived from~\eqref{eq::ELBO}:
\begin{align}
\label{eq::ELBO_updated_1}
      \mathcal{L}(x_{0:T}) &= \int q(x_0) \log p(x_0) dx_0 \nonumber \\
      &~~~~+ \sum\nolimits_{t=1}^T \int q(x_{t:t-1})  \log p(x_t|x_{t-1}) dx_{t:t-1} \nonumber \\
      &~~~~+ \sum\nolimits_{t=1}^T \int q(x_t)  \log p(z_t|x_t) dx_t \nonumber \\
      &~~~~- \int q(x_{0:T}) \log \biggl( q(x_0)\prod\nolimits_{t=1}^T q(x_t | x_{t-1}) \biggl) dx_{0:T} \nonumber \\
      &= \int q(x_0) \log p(x_0) dx_0 - \int q(x_0) \log q(x_0) dx_0 \nonumber \\
     &~~~~+ \sum\nolimits_{t=1}^T \int q(x_{t:t-1})  \log p(x_t|x_{t-1}) dx_{t:t-1} \nonumber \\
      &~~~~+ \sum\nolimits_{t=1}^T \int q(x_t)  \log p(z_t|x_t) dx_t  \\
   &~~~~-\! \sum\nolimits_{t=1}^T \!\!\int\! q(x_{t:t-1})  \log q(x_t | x_{t-1}) dx_{t:t-1}. \nonumber
\end{align}
Next, we apply the marginalization technique and the law of total probability to demonstrate that the ELBO in~\eqref{eq::ELBO_updated_1} can be represented \emph{recursively} as follows:
\begin{align}
\label{eq::ELBO_updated_2}
      &\mathcal{L}(x_{0:T}) = \int q(x_0) \log p(x_0) dx_0 - \int q(x_0) \log q(x_0) dx_0 \nonumber \\
      &~~+ \int q(x_1|x_0)q(x_0) \log p(x_1|x_0) dx_{1:0} \nonumber \\
      &~~+ \sum\nolimits_{t=2}^T \int q(x_t|x_{t-1})q(x_{t-1}|x_{t-2})\log p(x_t|x_{t-1}) dx_{t:t-1} \nonumber \\
      &~~+ \int q(x_1 | x_0)q(x_0) \log p(z_1|x_1) dx_{1:0} \nonumber \\
      &~~+ \sum\nolimits_{t=2}^T \int q(x_t|x_{t-1})q(x_{t-1}|x_{t-2}) \log p(z_t|x_t) dx_{t:t-2} \nonumber \\
      &~~- \int q(x_1|x_0) q(x_0) \log q(x_1 | x_0) dx_{1:0}  \\
      &~~- \sum\nolimits_{t=2}^T \!\! \int q(x_t|x_{t-1}) q(x_{t-1}|x_{t-2}) \log q(x_t | x_{t-1}) dx_{t:t-1}. \nonumber
\end{align}
Each term is rearranged based on its dependence on $x_0$, $x_{1:0}$, and $x_{t:t-2}$, respectively. Then, under Assumption~\ref{assump:smoothness}, the ELBO can be expressed as:
\begin{align}
      &\mathcal{L}(x_{0:T}) = \int q(x_0) \biggl(\log p(x_0) - \log q(x_0)\biggl) dx_0 \nonumber \\
      &+ \int \biggl( \int q(x_1|x_0) \Bigl( \log p(x_1|x_0) + \log p(z_1|x_1) \nonumber \\
      &\qquad\qquad\qquad\qquad\qquad - \log q(x_1 | x_0) \Bigl) dx_1 \biggl) q(x_0) dx_0 \nonumber \\
      &+ \sum\nolimits_{t=2}^T \int \biggl( \int q(x_t|x_{t-1}) \Bigl(\log p(x_t|x_{t-1}) + \log p(z_t|x_t) \nonumber \\
      &\qquad - \log q(x_t | x_{t-1}) \Bigl) dx_t \biggl) q(x_{t-1}|x_{t-2}) dx_{t-1}.
\end{align}
Therefore, by applying the definition of KL divergence, the ELBO becomes~\eqref{eq::ELBO_factorized}.

\subsection{Proof of Lemma~\ref{lem::conditional_svgd}} \label{sec::appen_D}
We extend the result of~\cite{liu2016stein} to the marginalized conditional KL objective. Define the perturbed distribution as $q_\epsilon(x_t | x_{t-1}) = \frac{1}{N_s} \sum\nolimits_{i=1}^{N_s} \delta\Bigl(x_t - (x_t^i + \epsilon \, \hat{\phi}^{\star}(x_t^i))\Bigl)$, where the perturbation follows the update in~\eqref{eq::SVGD_newupdate}. Under Assumption~\ref{assump::MAP}, Lemma~\ref{lem::proposal}, and~\ref{lem::factKL}, the corresponding marginalized KL divergence objective becomes
\begin{align}
    \!\!\mathcal{J}_t(q_\epsilon) \!\!=\!\!\! \int \!\!\text{KL} \bigr[q_\epsilon(x_t|x_{t-1}) \!\!\parallel\! p(z_t, x_t|x_{t-1})\!\bigr] q(x_{t-1}|x_{t-2}) dx_{t-1}\!.\!\! \nonumber
\end{align}
We compute its directional derivative at $\epsilon = 0$ as follows
\begin{align} \label{eq::D1}
    &\frac{d}{d\epsilon} \mathcal{J}_t(q_\epsilon) \Bigl|_{\epsilon = 0} \\
    &=\!\!\! \int \!\!\frac{d}{d\epsilon}\text{KL} \bigr[q_\epsilon(x_t|x_{t-1}) \!\!\parallel\! p(z_t, x_t|x_{t-1})\!\bigr]\!\Bigl|_{\epsilon = 0} \! q(x_{t-1}|x_{t-2}) dx_{t-1}\!.\!\! \nonumber
\end{align}
Using Stein variational theory, the directional derivative of the KL divergence for fixed $x_{t-1}$ is
\begin{align} \label{eq::D2}
    &\frac{d}{d\epsilon}
    \text{KL}\left[q_\epsilon(x_t|x_{t-1}) \parallel p(z_t, x_t|x_{t-1})\right]
    \Bigl|_{\epsilon = 0}  \\
    &=\! - \mathbb{E}_{q(x_t|x_{t-1})}
    \Bigl[ \phi(x_t)^\top\nabla_{x_t} \log p(z_t, x_t|x_{t-1}) \!+\! \nabla\! \cdot \phi(x_t) \Bigl], \nonumber
\end{align}
where $\nabla\! \cdot \phi(x_t)$ is the divergence of $\phi(x_t)$. Then, substituting~\eqref{eq::D2} into the expression for $\mathcal{J}_t(q_\epsilon)$ in~\eqref{eq::D1}, we obtain
\begin{align}
    &\frac{d}{d\epsilon} \mathcal{J}_t(q_\epsilon) \Bigl|_{\epsilon = 0} \\
    &= \!-\!\! \int\!\! \mathbb{E}_{q(x_t|x_{t-1})}\!
    \Bigl[ \phi(x_t)^\top\nabla_{x_t} \log p(z_t, x_t|x_{t-1}) \!+\! \nabla\! \cdot \phi(x_t) \Bigl] \nonumber \\
    &\qquad\qquad\qquad\qquad\qquad\qquad\qquad~ \times q(x_{t-1}|x_{t-2}) dx_{t-1}. \nonumber
\end{align}
Using the Stein operator $\mathcal{A}_p \phi(x_t)$, the derivative becomes
\begin{align}
    \frac{d}{d\epsilon} \mathcal{J}_t(q_\epsilon) \Bigl|_{\epsilon = 0}
    =\! - \mathbb{E}_{q(x_{t-1}| x_{t-2})}
    \Bigl[ \mathbb{E}_{q(x_t|x_{t-1})}
    \bigl[ \operatorname{tr}(\mathcal{A}_p \phi(x_t)) \bigl] \Bigl].
\end{align}
where $\mathcal{A}_p \phi(x_t)
\!:=\! \nabla_{x_t} \!\log p(z_t, x_t|x_{t-1}) \phi(x_t)^\top \!\!+\! \nabla_{x_t} \phi(x_t)$ and $\operatorname{tr}(\cdot)$ denotes the trace operator. Under Assumption~\ref{assump:smoothness} and the assumption of the boundedness of the kernel $\kappa(\cdot, \cdot)$, the SVGD direction $\phi^{\star} \in \mathcal{H}^d$ is chosen to maximize the negative directional derivative. Therefore,
\begin{align}
    \frac{d}{d\epsilon} \mathcal{J}_t(q_\epsilon) \Bigl|_{\epsilon = 0} < 0,
\end{align}
unless $q(x_t|x_{t-1}) = p(z_t, x_t|x_{t-1})$ almost surely under $q(x_{t-1}|x_{t-2})$.

\subsection{Proof of Theorem~\ref{thm::ELBO_MAP}} \label{sec::appen_E}

First, we prove that the ELBO in~\eqref{eq::ELBO_factorized} converges to $J(s)$ as $\varepsilon \to 0$ by using a standard property of mollifiers~\cite{brezis2011functional}. From~\eqref{eq::ELBO_factorized} in Lemma~\ref{lem::factKL}, we have
\begin{align}
    &\mathcal{L}_{\varepsilon}(q_{\varepsilon}) \!=\! \mathbb{E}_{q_{\varepsilon}(x_0)}[\log p(x_0)] \!+\! \mathbb{E}_{q_{\varepsilon}(x_0)}\mathbb{E}_{q_{\varepsilon}(x_1|x_0)}[\log p(z_1,x_1|x_0)] \nonumber \\
    &+ \sum_{t=2}^T\nolimits \mathbb{E}_{q_{\varepsilon}(x_{t-1}|x_{t-2})}\mathbb{E}_{q_{\varepsilon}(x_{t}|x_{t-1})}[\log p(z_t,x_t|x_{t-1})].\!\!   
\end{align}

By Assumption~\ref{assump:smoothness} and using a Gaussian as a mollifier (approximate identity)~\cite{stein2003fourier}, $\mathbb{E}_{q_{\varepsilon}(\cdot)}[\log p] \rightarrow \log p(s)$ as $\varepsilon \rightarrow 0$. Applying this to each term above yields
\begin{align} \label{eq::appx_E_MAP}
    &\lim_{\varepsilon \to 0}\mathcal{L}_{\varepsilon}(q_{\varepsilon}) \nonumber \\ 
    &= \log p(s_0) + \log p(z_1,s_1|s_0) 
       + \sum_{t=2}^T\nolimits \log p(z_t,s_t|s_{t-1}) \nonumber \\
    &= J(s)
\end{align}

Second, we prove that $\mathcal{L}_{\varepsilon}(s^{\star})$ is the maximizer over $\mathcal{S}$. Since $\mathcal{S}$ is finite and there exists a unique $s^{\star} \in \arg\max_{s\in \mathcal{S}}J(s)$ with gap $\delta =  J(s^{\star}) - \max_{s \neq s^{\star}} J(s) > 0$, it follows from~\eqref{eq::appx_E_MAP} that
\begin{align}
    M_{\varepsilon} := \max_{s \in \mathcal{S}} 
    \big| \mathcal{L}_{\varepsilon}(s) - J(s) \big| \xrightarrow[\varepsilon \to 0]{} 0.
\end{align}

Choose $\varepsilon_0 > 0$ such that 
$M_{\varepsilon} < \delta/3$ for all $0 < \varepsilon < \varepsilon_0$. Then, for such $\varepsilon$,
\begin{align}
    &\mathcal{L}_{\varepsilon}(s^{\star}) \ge J(s^{\star}) - \frac{\delta}{3}, \nonumber \\
    &\mathcal{L}_{\varepsilon}(s) 
    \le J(s) + \frac{\delta}{3} 
    \le J(s^{\star}) - \delta + \frac{\delta}{3} 
    = J(s^{\star}) - \frac{2\delta}{3},~\forall\, s \neq s^{\star}\!\!. \nonumber
\end{align}
Therefore, $\mathcal{L}_{\varepsilon}(s^{\star}) > \mathcal{L}_{\varepsilon}(s), \forall s \neq s^{\star}$, hence, the maximizer of $\mathcal{L}_{\varepsilon}$ over $\mathcal{S}$ coincides with $s^{\star}$.

\subsection{Per-Step Runtime and Scalability Analysis} \label{sec::appen_F}
We examine the computational efficiency of~\eqref{eq::SVGD_newupdate} in \textsf{Stein-MAP-Seq} by measuring the average per-step runtime on a GPU as a function of the number of particles $N_s$ and the state dimension $n_x$. For consistency with all experimental scenarios, a fixed number of SVGD update iterations per time step is used ($100$ iterations). This analysis illustrates how the computational cost grows as the number of particles and the state dimension increase, thereby providing practical evidence that \textsf{Stein-MAP-Seq} remains computationally viable in high-dimensional estimation problems.

\begin{figure}[ht]
    \centering
    \begin{minipage}[t]{0.22\textwidth}
        \centering
        \includegraphics[width=\textwidth]{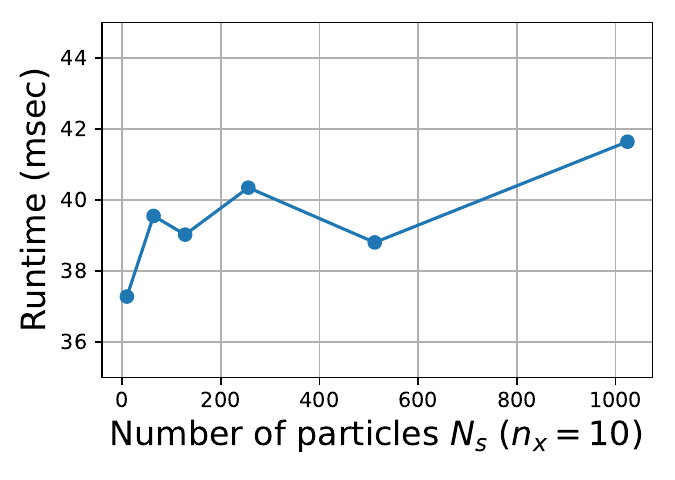}
        {{\scriptsize (a) Effect of Particle Count on Per-Step Runtime}}
    \end{minipage}
    \hskip\baselineskip
    \begin{minipage}[t]{0.22\textwidth}
        \centering
        \includegraphics[width=\textwidth]{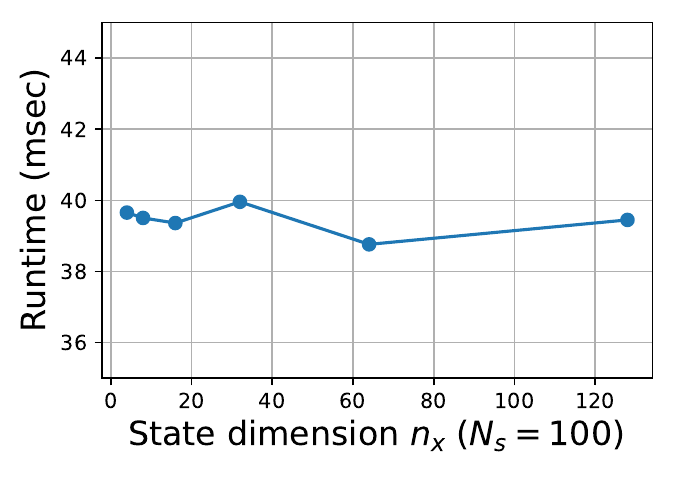}
        {{\scriptsize (b) Effect of State Dimension on Per-Step Runtime}}
    \end{minipage}
    \caption{{\small Per-step computational cost of \textsf{Stein-MAP-Seq}, averaged over Monte Carlo runs, measured using a fixed number of SVGD update iterations per time step (100 iterations). The results illustrate stable runtime behavior across varying estimation problem sizes.}
    }
    \label{fig:appendix:com}
\end{figure}

\bibliographystyle{ieeetr}
\bibliography{Bib/Reference.bib}

\end{document}